\documentclass[twoside,11pt]{article}
\usepackage[shortlabels]{enumitem}
\usepackage{amsmath}
\usepackage{amsfonts}
\usepackage{Definitions}
\usepackage{natbib}
\usepackage{color}
\usepackage[colorlinks,
            linkcolor=blue,
            anchorcolor=blue,
            citecolor=green
            ]{hyperref}

\usepackage{algorithm}
\usepackage{algorithmic}
\usepackage{fullpage}

\usepackage{xcolor}
\usepackage{xspace}
\usepackage{graphicx}
\usepackage{subcaption}
\usepackage{wrapfig}
\usepackage{booktabs}
\usepackage{multirow}
\usepackage{multicol}

\usepackage{lastpage}

\providecommand{\keywords}[1]
{
  \textbf{{Keywords---}} #1
}

\title{Spectral Representation-based Reinforcement Learning}
\author{%
    \normalsize
  \begin{tabular}{c@{\hspace{2em}}c@{\hspace{2em}}c} 
    \textbf{Chenxiao Gao} & \textbf{Haotian Sun} & \textbf{Na Li}\\
    Georgia Tech & Georgia Tech & Harvard University\\
    \texttt{cgao@gatech.edu} & \texttt{hsun409@gatech.edu} & \texttt{nali@seas.harvard.edu}\\
  \end{tabular}
  \vspace{1em}
  \\
  \normalsize
  \begin{tabular}{c@{\hspace{1em}}c} 
    \textbf{Dale Schuurmans} & \textbf{Bo Dai} \\
    Google DeepMind \& University of Alberta & Google DeepMind \& Georgia Tech \\
    \texttt{schuurmans@google.com} & \texttt{bodai@google.com}
  \end{tabular}%
}

\begin{document}
\maketitle

\begin{abstract}
In real-world applications with large state and action spaces, reinforcement learning (RL) typically employs function approximations to represent core components like the policies, value functions, and dynamics models. Although powerful approximations such as neural networks offer great expressiveness, they often present theoretical ambiguities, suffer from optimization instability and exploration difficulty, and incur substantial computational costs in practice. In this paper, we introduce the perspective of \emph{spectral representations} as a solution to address these difficulties in RL. Stemming from the spectral decomposition of the transition operator, this framework yields an effective abstraction of the system dynamics for subsequent policy optimization while also providing a clear theoretical characterization. We reveal how to construct spectral representations for transition operators that possess latent variable structures or energy-based structures, which implies different learning methods to extract spectral representations from data. Notably, each of these learning methods realizes an effective RL algorithm under this framework. We also provably extend this spectral view to partially observable MDPs. Finally, we validate these algorithms on over 20 challenging tasks from the DeepMind Control Suite, where they achieve performances comparable or superior to current state-of-the-art model-free and model-based baselines. Our code is publicly released at \url{https://spectral-rl.github.io}.
\end{abstract}


\newpage
\tableofcontents

\newpage

\section{Introduction}

Reinforcement Learning (RL) has emerged as a powerful tool in real-world applications that involve sequential decision-making \citep{degrave2022magnetic,ouyang2022training,tang2025deep}. It formulates tasks based on Markov Decision Processes (MDPs) and learns an optimal policy that can maximize the expected cumulative return through interacting with an unknown environment \citep{sutton1998reinforcement}. For MDPs with finite state and action spaces, known as the tabular case, efficient and convergent algorithms with well-defined theoretical sample complexities have been developed \citep{auer2008near,dann2015sample,agrawal2017optimistic,azar2017minimax,jin2018q}. However, the optimization cost of these algorithms quickly becomes unacceptable as the volume of the state space and action space grows beyond tabular cases, even towards infinity in continuous state-action cases. To tackle such scalability challenges, modern RL algorithms employ function approximations to parameterize the core components in the pipeline \citep{jin2020provably,yang2020function,long2021perturbational}. Based on which component is approximated, these algorithms can be roughly classified into two categories: \textit{model-free RL} and \textit{model-based RL}. Model-free methods directly parameterize the policy and/or value functions, which are then optimized in an end-to-end fashion with trial-and-error interactions with the environment \citep{dai2018sbeed,fujimoto2018addressing,haarnoja2018soft}. In contrast, model-based RL trains a surrogate dynamics model to capture the environment's state transition, and the optimal policy can then be derived by planning or model predictive control using this learned model \citep{hafner2019dream,janner2019trust}. 

The architectural simplicity of modern model-free RL has enabled its successful integration with powerful function approximations (\eg, deep neural networks) for policy/value function parameterization, leading to notable applications in diverse domains, such as generative model fine-tuning \citep{ouyang2022training,yang2024using}, video games \citep{mnih2013playing,berner2019dota}, and robotic control \citep{tang2025deep}. However, such expressive power also comes with theoretical challenges: with general function approximations, standard RL methods such as temporal-difference learning with off-policy data can diverge under some circumstances, a phenomenon known as the ``deadly triad" \citep{sutton1998reinforcement}. Furthermore, model-free RL is completely blind to the underlying dynamics, 
ultimately leading to a prohibitive sample complexity in practice \citep{fujimoto2018addressing,haarnoja2018soft}. 

\begin{figure}[t]
    \centering
    \includegraphics[width=0.95\linewidth]{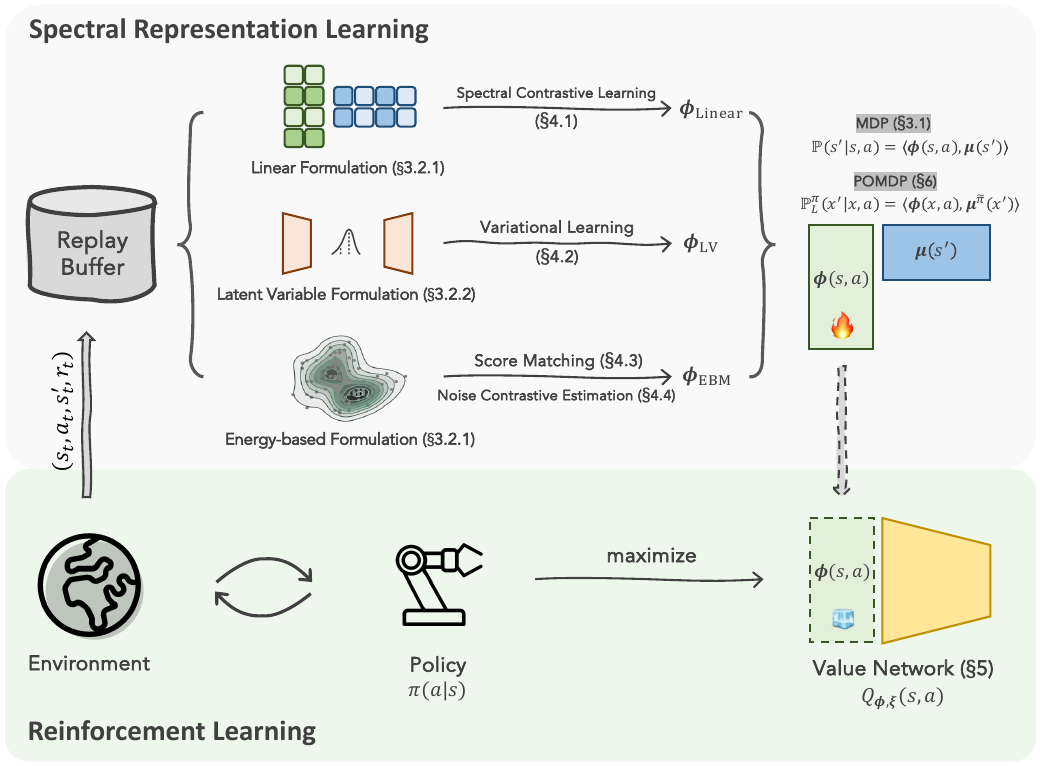}
    \caption{The overview of spectral representation-based reinforcement learning.}
    \label{fig:banner}
\end{figure}

Model-based RL, on the other hand, explicitly learns a model of the environment's dynamics and leverages this model to facilitate policy optimization \citep{wang2019benchmarking,luo2024survey}. When the model possesses a certain structure (\eg, linear quadratic regulators (LQRs)), it can be shown that optimal policies can be derived efficiently. However, in complex domains where the state transition can be nonlinear or non-deterministic, advanced modeling techniques (\eg, recurrent \citep{hafner2019learning,hafner2019dream}, attentional \citep{chen2022transdreamer,micheli2022transformers}, and diffusion-based models \citep{ding2024diffusion,alonso2024diffusion}) with larger capacity become necessary in order to accurately reflect the true dynamics and reduce the approximation error. This introduces a dilemma analogous to that faced in model-free RL: with non-linear models, solving for an optimal policy via planning is computationally intractable, and becoming a long-standing challenging problem in control community~\citep{khalil2015nonlinear}.
Consequently, practical model-based RL must resort to approximate inference methods, such as Dyna-style planning \citep{kurutach2018model,janner2019trust,luo2024survey} or backpropagating the gradient through the models \citep{deisenroth2011pilco,heess2015learning}. These compromises not only risk converging to suboptimal solutions, but also incur substantial costs both during training and inference due to their intricate designs, and thus waste the flexibility in modeling and effort in learning. 

In conclusion, a persistent challenge pervades both model-based and model-free reinforcement learning: there is a lack of unified frameworks that can be simultaneously characterized by rigorous statistical guarantees and realized as computationally efficient, practical algorithms. This fundamental challenge leads to a very natural question: 
\begin{center}
    \emph{Can we develop both provably efficient and practically effective algorithms for reinforcement learning? }
\end{center}

Motivated by the recent success of representation learning, we present \textit{spectral representations} as an affirmative answer to the above question. Specifically, spectral representations stem from the functional decomposition of the transition probability operator in MDP, thereby encoding sufficient information about the underlying dynamics. Leveraging an analysis analogous to the linear MDP literature \citep{jin2020provably}, it can be demonstrated that the $Q$-value functions in RL can be sufficiently expressed with spectral representations, which paves the way for efficient learning and planning. Although several algorithms have been developed to leverage the power of spectral representations for reinforcement learning \citep{ren2022free,zhang2022making,ren2023latent,ren2023spectral,zhang2023provable,shribak2024diffusion}, this field nevertheless lacks a coherent framework to unify these methods, elucidate their connections, and conduct fair comparison. This survey aims to fill this blank by organizing these disparate methods into a cohesive overview and systematically benchmarking them through controlled experiments. 

We illustrate the structure of this survey in~\figref{fig:banner}.
Specifically, after introducing the background knowledge in~\secref{sec:background}, \secref{sec:spectral_representations} begins with a comprehensive overview of the spectral representation, covering its definitions and properties ({\secref{sec:spectral_view_of_mdps}}) and its construction with different dynamics structures (\secref{sec:linear_nonlinear_spectral_representation}). To address the intractability of spectral decomposition and maximum likelihood estimation (MLE) in large-scale spaces, \secref{sec:learning_spectral_representations} reviews practical algorithms for learning these representations directly from data. Next, \secref{sec:rl_with_spectral_rep} details how to integrate these learned representations into the reinforcement learning pipeline. \secref{sec:pomdps} extends the framework to partially observable Markov decision processes (POMDPs), thereby accommodating more realistic decision-making scenarios. Finally, to validate the effectiveness as well as to facilitate further research in this direction, \secref{sec:experiments} presents an extensive empirical study. We release a unified codebase and benchmark performance on over 20 continuous control tasks from the DeepMind Control Suite \citep{tassa2018deepmind}, encompassing diverse settings that include both fully observable state-based inputs and partially observable visual data. Generally, the algorithms based on spectral representations outperform their model-free counterparts, especially as task dynamics become more complex. In visual settings, these methods achieve performance comparable or superior to state-of-the-art model-based algorithms while obviating the need for computationally expensive trajectory synthesis. Finally, through a detailed ablation analysis, we investigate the influence of certain key design choices to provide a deeper understanding of these methods.

\section{Preliminaries}\label{sec:background}
We formulate the task as a Markov Decision Process (MDP) specified by a tuple $\Mcal=\langle \Scal, \Acal, \PP, r , \gamma, d_0\rangle$, where $\Scal$ is the state space, $\Acal$ is the action space, $\PP:\Scal\times\Acal\rightarrow \Delta(\Scal)$ specifies the transition probability of states, $r :\Scal\times\Acal\rightarrow \RR$ is the instantaneous reward function, $\gamma\in[0, 1)$ is the discounting factor and $d_0\in\Delta(\Scal)$ is the initial state distribution. The agents start with an initial state $s_0\sim d_0$ and select an action $a_t\sim\pi(\cdot | s_t)$ following its policy $\pi:\Scal\to\Delta(\Acal)$. The environment then takes $a_t$,  transits to a new state $s_{t+1}$, and provides scalar feedback $r_t=r (s_t, a_t)$ to the agent. RL seeks to optimize the policy to maximize its expected cumulative return, defined as:
\begin{equation}
    \begin{aligned}
        \Jcal(\pi,\Mcal):=\EE_{a_t\sim \pi(\cdot|s_t),s_{t+1}\sim \PP(\cdot|s_t, a_t)}\left[\sum_{t=0}^\infty\gamma^tr (s_t, a_t)\Big|s_0\sim d_0\right].
    \end{aligned}
\end{equation}
To evaluate the values of each state and action, we define the state value functions $V^\pi: \Scal \to \RR$ and state-action value functions $Q^\pi: \Scal\times\Acal\to \RR$ of a given policy $\pi$ as:
\begin{equation}
    \begin{aligned}
        V^\pi(s)&:=\EE_{a_t\sim \pi(\cdot|s_t), s_{t+1}\sim \PP(\cdot|s_t, a_t)}\left[\sum_{t=0}^\infty\gamma^t r (s_t, a_t)\Big| s_0=s\right],\\
        Q^\pi(s, a)&:=\EE_{a_t\sim \pi(\cdot|s_t), s_{t+1}\sim \PP(\cdot|s_t, a_t)}\left[\sum_{t=0}^\infty\gamma^t r (s_t, a_t)\Big| s_0=s, a_0=a\right].
    \end{aligned}
\end{equation}
$V^\pi$ and $Q^\pi$ satisfy the following recursion:
\begin{equation}\label{eq:bellman_equation}
    \begin{aligned}
        Q^\pi(s, a)&=r(s, a)+\gamma \EE_{s_{t+1}\sim \PP(\cdot|s_t, a_t)}\left[V^\pi(s_{t+1})\right]\\
        &=r(s, a)+\gamma \EE_{s_{t+1}\sim \PP(\cdot|s_t, a_t), a_{t+1}\sim\pi(\cdot|s_{t+1})}\left[Q^\pi(s_{t+1}, a_{t+1})\right],
    \end{aligned}
\end{equation}
which is also known as the Bellman equation. With the help of value functions, the objective of RL can be conveniently defined as finding $\pi^*=\argmax_{\pi}\ \EE_{s\sim d_0}\left[V^\pi(s)\right]=\argmax_{\pi}\ \EE_{s\sim d_0, a\sim \pi}\left[Q^\pi(s, a)\right]$.

Finally, we define the \emph{occupancy measure} $d^\pi(\cdot, \cdot)\in\Delta(\Scal\times\Acal)$ as the normalized discounted probability of visiting $(s, a)$ under policy $\pi$:
\begin{equation}
    \begin{aligned}
        d^\pi(s, a):=(1-\gamma)\EE_{\pi,\PP}\left[\sum_{t=0}^\infty\gamma^t\II\{s_t=s, a_t=a\}\right].
    \end{aligned}
\end{equation}
Based on the stationary condition, the occupancy measure satisfies the following recursion:
\begin{equation}\label{eq:recursion_occupancy}
    \begin{aligned}
        d^\pi(s, a)&=(1-\gamma)d_0(s)\pi(a|s)+\gamma\int d^\pi(\tilde{s}, \tilde{a})\PP^\pi(s, a|\tilde{s}, \tilde{a})\mathrm{d}\tilde{s}\mathrm{d}\tilde{a},
    \end{aligned}
\end{equation}
where $\PP^\pi(s, a|\tilde{s},\tilde{a})$ is a shorthand for $\PP(s|\tilde{s},\tilde{a})\pi(a|s)$. 

\section{The Framework of Spectral Representations}\label{sec:spectral_representations}

This section introduces a spectral perspective on the transition dynamics of MDPs by formalizing them through a functional decomposition within a Hilbert space.
Crucially, such linear structure inspires a family of representations, which we term \textit{spectral representations}, that can capture the characteristics of the system dynamics and allow for sufficient expression of $Q$-value functions (\secref{sec:spectral_view_of_mdps}). 
We then reveal how to construct these spectral representations for both linear and non-linear dynamics (\secref{sec:linear_nonlinear_spectral_representation}), demonstrating the wide applicability of this framework. 

\subsection{Spectral View of MDPs}\label{sec:spectral_view_of_mdps}
While seminal work in reinforcement learning has proposed efficient algorithms for tabular MDPs with finite state and action spaces \citep{auer2008near,dann2015sample,agrawal2017optimistic}, scaling to large or continuous spaces has remained a challenge. Recent works have turned to function approximations to overcome such challenges. A prominent approach in this direction is the linear (low-rank) MDP hypothesis \citep{jin2020provably}, which posits that the transition operator admits a low-rank linear decomposition. Inspired by this, we consider a more generalized version without the finite rank assumption. Specifically, for any well-behaved transition operator $\PP$ and reward function $r$, there always exist two feature maps $\phib: \Scal\times\Acal\to \Hcal$, $\mub:\Scal\to\Hcal$ and a vector $\thetab_r\in\Hcal$ for some proper Hilbert space $\Hcal$, such that for $\forall (s,a)\in\Scal\times\Acal$, we have
    \begin{equation}\label{eq:spectral_mdp}
        \begin{aligned}
            \PP(s'|s, a) &=\langle\phib(s, a), \mub(s')\rangle_\Hcal,\\
            r(s, a)&=\langle \phib(s, a), \thetab_r\rangle_\Hcal.\\
        \end{aligned}
    \end{equation}
In fact, since the feature maps and the vector can be infinite-dimensional, such a decomposition always exists for the transition kernel $\PP$. To see this, the singular value decomposition (SVD) of the transition operator \citep{mollenhauer2020singular} gives a concrete example:
\begin{equation}
    \begin{aligned}
        \PP(s'|s, a) = \sum_{i\in I}\sigma_iu_i(s, a)v_i(s'),
    \end{aligned}
\end{equation}
where $I$ is an finite or countably infinite ordered index set, $\{u_i\}_{i\in I}$ and $\{v_i\}_{i\in I}$ are two orthogonal systems of the function spaces on $\Scal\times\Acal$ and $\Scal$, respectively. By expressing this summation as an inner product, we obtain a realization of the aforementioned linear decomposition. Furthermore, commonly adopted structures in the literature are instantiations of the decomposition in \eqref{eq:spectral_mdp} with additional assumptions. For example, linear MDP assumes $\Hcal=\RR^d$ for some finite $d$. Tabular MDP is also a special case, where $\Hcal=\RR^d$ with $d=|\Scal|\times|\Acal|$ and the feature map $\phib(s, a)$ is the canonical basis vector $\boldsymbol{e}_{(s, a)}\in\RR^d$. 

A key consequence of this linear structure is that the $Q$-value functions of any given policy $\pi$ also inherit this linearity \citep{yao2014pseudo,jin2020provably,yang2020reinforcement}.  
\begin{lemma}\label{lemma:q_sufficiency}
    For MDPs with \eqref{eq:spectral_mdp} and any policy $\pi$, there exists weights $\boldsymbol{\eta}^\pi\in\Hcal$ such that $\forall (s,a)\in\Scal\times\Acal$, $Q^\pi(s, a) = \inner{\phib(s, a)}{ \boldsymbol{\eta}^\pi}_\Hcal$. 
\end{lemma}
\begin{proof}
    The linearity of $Q$-value functions can be examined by the Bellman equation:
    \begin{equation}
        \begin{aligned}
            Q^\pi(s, a) &= r(s, a)+\gamma \int_\Scal \PP(s'|s, a)V(s')\mathrm{d}s'\\
            &=\inner{\phib(s, a)}{\thetab_r}_\Hcal+\gamma \inner{\phib(s, a)}{\int_\Scal \mub(s')V(s')\mathrm{d}s'}_\Hcal\\
            &=\bigg\langle\phib(s, a),\underbrace{\thetab_r+\gamma \int_\Scal \mub(s')V(s')\mathrm{d}s'}_{\boldsymbol{\eta}^\pi}\bigg\rangle_\Hcal.\\
        \end{aligned}
    \end{equation}
\end{proof}
That is, $Q$-value functions lie in the linear function space spanned by the same feature map $\phib(s, a)$. This property is significant for two reasons. First, it simplifies theoretical analysis by allowing us to focus on linear $Q$-value functions, thereby reducing the complexity and facilitating uncertainty quantification. Second, it suggests that $\phib$ can serve as a representation for efficient algorithm design, as we can first extract $\phib$ from the dynamics and only learn the rest part (\ie, $\boldsymbol{\eta}^\pi$) with reinforcement learning. Given that $\phib(s, a)$ spans the same space as the orthogonal system from SVD, we term such representations as \emph{spectral representations}. 

\paragraph{Remark (Connection to Existing Representations from Decomposition View)} To better situate the proposed spectral representations in the vast representation learning literature, we summarize existing representations that stem from spectral decomposition of transition-related matrices in Table \ref{tab:rep_compare}. As discussed, spectral representations are derived from the left singular vectors of the single-step transition operator. In contrast, successor features \citep{dayan1993improving, kulkarni2016deep,barreto2017successor} are actually the left singular vector of $(I-\gamma\PP^\pi)^{-1}$, as proved in~\citep{ren2023spectral}, which summarizes future state visitation from given states under a prefixed policy. The Laplacian-based methods \citep{wu2018laplacian} compute representations from the eigenvectors of the symmetrized transition matrix with the largest eigenvalues. Finally, if the reward function is known a priori, the Krylov basis \citep{petrik2007analysis} suggests using the series $\{(\PP^\pi)^ir\}_{i=1}^k$ as the representation for value functions. 
\begin{table}[thb]
\caption{A unified matrix decomposition perspective of various representations. In this table, $\texttt{svd}$ and $\texttt{eig}$ denotes SVD and eigen decomposition, $\PP^\pi$ denotes the state transition $\PP(s'|s)$ 
under policy $\pi$, and $\Lambda$ is a diagonal matrix.}\label{tab:rep_compare}
\centering
\begin{tabular}{cc}\\
\toprule  
Representation & Decomposition \\
\midrule
Spectral Representation & $\texttt{svd}(\PP)$ \\
Successor Feature \citep{dayan1993improving} & $\texttt{svd}\left((I-\gamma \PP^\pi)^{-1}\right)$\\
Laplacian \citep{wu2018laplacian}& $\texttt{eig}\left(\Lambda \PP^\pi + (\PP^\pi)^\top \Lambda\right)$\\
Krylov Basis \citep{petrik2007analysis} & $\{(\PP^\pi)^ir\}_{i=1}^k$\\
\bottomrule
\end{tabular}
\end{table}

Besides connecting the spectral representation to existing representations through the decomposition lens, more importantly, such a view also reveals the advantages of spectral representation over existing representation. 
Existing spectral representations are derived from $\PP^\pi$ (the policy-specific transition operator), which inherently introduces inter-state dependency tied to the specific policy $\pi$. This results in an unnecessary inductive bias in the state-only spectral features, potentially hindering generalizability across different policies. Furthermore, these methods entirely neglect the exploration dilemma: the policy influences the data composition used for representation learning, yet the representations simultaneously affect policy optimization. This coupling significantly complicates the process.
In contrast, the spectral representation investigated in this paper sidesteps the undesired dependency on $\pi$ by extracting state-action representations from $\PP\rbr{s'|s, a}$. 


\paragraph{Remark (Primal-Dual Spectral Representations)} With a fixed policy $\pi$, if we consider the spectral decomposition of the state-action transition kernel and the initial distribution:
\begin{equation}
    \begin{aligned}
        d_0(s)&=\inner{\thetab_d}{\mub(s)}_\Hcal,\\
        \PP^\pi(s', a'|s, a) &=\PP(s'|s, a)\pi(a'|s')=\Big\langle\phib(s, a),\underbrace{\mub(s')\pi(a'|s')}_{\mub^\pi(s', a')}\Big\rangle_\Hcal,
    \end{aligned}
\end{equation}
then, according to the recursion in \eqref{eq:recursion_occupancy}, the occupancy measure also admits the linear structure over the other side of the factor $\mub^\pi(s, a)$:
\begin{equation}
    \begin{aligned}
        d^\pi(s, a) &=(1-\gamma)d_0(s)\pi(a|s)+\gamma\int d^\pi(\tilde{s}, \tilde{a})\PP^\pi(s, a|\tilde{s}, \tilde{a})\mathrm{d}\tilde{s}\mathrm{d}\tilde{a}\\
        &=(1-\gamma)\inner{\thetab_d}{\mub^\pi(s, a)}_\Hcal + \gamma \int d^\pi(\tilde{s}, \tilde{a})\inner{\phib(\tilde{s}, \tilde{a})}{\mub^\pi(s, a)}_\Hcal\mathrm{d}\tilde{s}\mathrm{d}\tilde{a}\\
        &=\Bigg\langle\underbrace{(1-\gamma)\thetab_d+\gamma\int d^\pi(\tilde{s}, \tilde{a})\phib(\tilde{s}, \tilde{a})\mathrm{d}\tilde{a}}_{\boldsymbol{\eta}^\pi}\ ,\ \mub^\pi(s, a)\Bigg\rangle_\Hcal.
    \end{aligned}
\end{equation}
Exploiting this structure, \citet{hu2024primal} reduces the saddle-point optimization of the off-policy evaluation problem to a convex-concave objective, using $\phib(s, a)$ as the primal representation for $Q$-value functions and $\mub^\pi(s, a)$ as the dual representation for the occupancy measure.

\subsection{Identifying Spectral Representations from Dynamics}\label{sec:linear_nonlinear_spectral_representation}

However, although the decomposition \eqref{eq:spectral_mdp} holds in general, the system dynamics do not necessarily admit this form. In this section, we will examine popular models of system dynamics and reveal how to derive their corresponding spectral representations. Our discussion begins with the most straightforward case of linear dynamics, and finally extends to complex non-linear dynamics, such as latent-variable models and energy-based models. For the ease of presentation, we will only consider full-observable MDPs in this section.

\subsubsection{Linear Formulation}\label{sec:linear}
Consider the case where the system dynamics is linear, \ie, there exists $\varphib: \Scal\times\Acal\to\RR^d$ and $\nub: \Scal\to\RR^d$ such that
\begin{equation}\label{eq:linear_spectral_representation}
    \begin{aligned}
        \PP_\linear(s'|s, a) &=\varphib(s, a)^\top\nub(s'), 
    \end{aligned}
\end{equation}
By comparing this linear formulation \eqref{eq:linear_spectral_representation} with \eqref{eq:spectral_mdp}, we instantly identify the spectral representation as $\phib(s, a)=\varphib(s, a)\in\RR^d$. This is theoretically equivalent to the linear MDP hypothesis in \citet{jin2020provably}; however, we note that this structure is limited, as the finite-rank linearity assumption is strong and restricts the model's expressive capacity.

\subsubsection{Latent Variable Formulation}\label{sec:lv}

Consider a transition operator $\PP_{\text{LV}}$ with a latent variable structure. Assuming the latent variable is $z\in\Zcal$ and there exist two probability measures $\varphi(z|s, a)$ and $\nu(s'|z)$ such that the latent variable dynamics model can be expressed as
\begin{equation}
        \PP_{\text{LV}}(s'|s, a) = \int_\Zcal \varphi(z|s,a)\nu(s'|z)\mathrm{d}z, 
\end{equation}
where $\varphi(\cdot|s,a): \Scal\times\Acal\rightarrow \Delta\rbr{\Zcal}$ and $\nu\rbr{\cdot|z}: \Zcal\rightarrow \Delta\rbr{\Scal}$. Notably, if the measures $\varphi(\cdot|s, a)\in L_2$ and $\nu(s'|\cdot)\in L_2$, the integral in the latent variable formulation can also be formulated as an inner product \citep{ren2023latent}:
\begin{equation}\label{eq:lvrep}
        \PP_\lv(s'|s, a)=\inner{\varphi(\cdot | s, a)}{\nu(s'|\cdot)}_{L_2},
\end{equation}
which instantly follows the decomposition defined in \eqref{eq:spectral_mdp} and recognizes the measure $\varphi(\cdot|s, a)$ as the spectral representation, denoted as $\phib_{\lv}$. 
The representation $\phib_\lv$ 
can be finite or infinite dimensional for discrete or continuous latent variable $z$, respectively. 

In fact, this structure is widely adopted in model-based reinforcement learning and world models with varied specific choices for the latent space and distributions. For instance, setting $\Zcal=\RR^d$ and using Gaussian distributions for $p$ recovers the dynamics modeling used in PlaNet \citep{hafner2019learning} and Dreamer \citep{hafner2019dream}; while setting $Z=\{1, 2, \ldots, D\}$ and $p(\cdot|s, a)$ as Categorical distributions recovers the models in DreamerV2 \citep{hafner2020mastering} and DreamerV3 \citep{hafner2023mastering}. Moreover, such latent dynamics can also be built upon learned state representations, as demonstrated by \citet{zhou2411dino} and \citet{xiang2025pan}.

\paragraph{Remark (Theoretical limitations of $\phib_\lv$).} Since $\varphi(\cdot|s, a)$ must be a valid probability measure, latent variable models are thus limited in expressiveness as compared to the unconstrained version. In fact, \citet{agarwal2020flambe} has demonstrated that the simplex feature $\phib_\lv$ can be exponentially weaker than the linear counterpart $\phib_\linear$, meaning that $\phib_\lv$ may require exponentially more dimensions than $\phib_\linear$ to characterize the same transition operator. However, as the theoretical characterization in \secref{sec:rl_with_spectral_rep} suggests, the apparent dimension may not be a good measure of the complexity of the representations. As long as the transition operator satisfies some regularity conditions, sampling-efficient learning is still possible.

\subsubsection{Energy-based Formulation}\label{sec:ebm}
As one of the most flexible model parameterizations, energy-based models (EBMs) associate the transition probability with an energy function $E(s, a, s')$: 
\begin{equation}
    \begin{aligned}
        \PP_\ebm(s'|s, a)&=\frac{\exp(E(s, a, s'))}{Z(s, a)}=\frac{\exp(\varphib(s, a)^\top \nub(s'))}{Z(s, a)},
    \end{aligned}
\end{equation}
where the $Z(s, a)$ is the normalizing factor satisfying $Z(s, a)=\int_\Scal\exp(E(s, a, s'))\mathrm{d}s'$, and the energy function can be factorized as the inner product of $\varphib: \Scal\times\Acal\to\RR^d$ and $\nub: \Scal\to\RR^d$. By simple algebra manipulation, 
\begin{equation}
    \begin{aligned}
        \varphib(s, a)^\top\nub(s')&=-\frac 12\left(\|\varphib(s, a) - \nub(s')\|^2 - \|\varphib(s, a)\|^2 - \|\nub(s')\|^2\right), 
    \end{aligned}
\end{equation}
and therefore,
\begin{equation}\label{eq:ebm1}
    \begin{aligned}
        \PP_\ebm(s'|s, a)&=\frac{\exp(\frac 12\|\varphib(s, a)\|^2)\exp(-\frac 12\|\varphib(s, a)-\nub(s')\|^2)\exp(\frac 12\|\nub(s')\|^2)}{Z(s, a)}.\\
    \end{aligned}
\end{equation}
In order to reveal the linear perspective of \eqref{eq:ebm1}, note that $\exp(\frac 12 \|\varphib(s, a)-\nub(s')\|)$ in \eqref{eq:ebm1} is precisely a Gaussian kernel, and therefore
\begin{equation}\label{eq:ebm2}
    \begin{aligned}
        \PP_\ebm(s'|s, a) &=\frac{\exp(\frac 12\|\varphib(s, a)\|^2)k\left(\varphib(s, a), \nub(s')\right)\exp(\frac 12\|\nub(s')\|^2)}{Z(s, a)}\\
        &=\inner{\frac{\exp(\frac 12\|\varphib(s, a)\|^2)}{Z(s, a)}k\left(\varphib(s, a), \cdot\right)}{k\left(\cdot, \nub(s')\right)\exp(\frac 12\|\nub(s')\|^2)}_{\Hcal_k}.
    \end{aligned}
\end{equation}
This final expression represents the transition as an inner product in $\Hcal_k$, the reproducing kernel Hilbert space of the Gaussian kernel $k$, and satisfies the decomposition in \eqref{eq:spectral_mdp}. However, this RKHS representation $k(\varphib(s, a), \cdot)$ and $k(\cdot, \nub(s'))$ is conceptual and implicit. To derive a concrete representation, \citet{ren2022free} and \citet{shribak2024diffusion} adopt random Fourier features \citep{rahimi2007random}, as defined in the following lemma.

\begin{lemma}{(Random Fourier features \citep{rahimi2007random})}
    Consider a shift invariant kernel $k(\xb, \yb)=k(\xb-\yb)$ on $\RR^d$. Define $\zeta_{\omegabb}(\xb)=e^{j\omegabb^\top \xb}$, we have
    \begin{equation}\label{eq:rff}
        \begin{aligned}
            k(\xb-\yb) &= \int_{\RR^d}p(\omegabb)e^{j\omegabb^\top(\xb-\yb)}\mathrm{d}\omegabb=\EE_{\omegabb\sim p(\omegabb)}\left[\zeta_{\omegabb}(\xb)\zeta_{\omegabb}(\yb)^*\right],\\
        \end{aligned}
    \end{equation}
    where $p(\omegabb)$ is the Fourier transform of the kernel $k(\Delta)$. Particularly, for $k(\Delta)=\exp(-\frac{\|\Delta\|_2^2}{2})$, then $p(\omegabb)=(2\pi)^{-\frac d2}\exp(-\frac{\|\omegabb\|_2^2}{2})$.
\end{lemma}

The expectation in \eqref{eq:rff} can be further approximated with Monte-Carlo samples of $\{\omegabb_i\}_{i=1}^N\sim p(\omegabb)$. Besides, since we are working with real-valued transition probabilities, we can replace $\zeta_{\omegabb}(\xb)$ with $[\cos(\omegabb^\top\xb), \sin(\omegabb^\top\xb)]^\top$. This leads to the final approximation of $k(\xb, \yb)$:
\begin{equation}
    \begin{aligned}
        k(\xb - \yb) &= \EE_{\omegabb\sim p(\omegabb)}\left[\zeta_{\omegabb}(\xb)\zeta_{\omegabb}(\yb)^*\right]\\
        &\approx\begin{pmatrix}
            \frac 1{\sqrt{N}}\cos(\omegabb_1^\top\xb)\\
            \frac 1{\sqrt{N}}\sin(\omegabb_1^\top\xb)\\
            \ldots\\
            \frac 1{\sqrt{N}}\cos(\omegabb_N^\top\xb)\\
            \frac 1{\sqrt{N}}\sin(\omegabb_N^\top\xb)\\
        \end{pmatrix}^\top\begin{pmatrix}
            \frac 1{\sqrt{N}}\cos(\omegabb_1^\top\yb)\\
            \frac 1{\sqrt{N}}\sin(\omegabb_1^\top\yb)\\
            \ldots\\
            \frac 1{\sqrt{N}}\cos(\omegabb_N^\top\yb)\\
            \frac 1{\sqrt{N}}\sin(\omegabb_N^\top\yb)\\
        \end{pmatrix}\\
        &=\langle \zetab_N(\xb), \zetab_N(\yb)\rangle_{L_2},
    \end{aligned}
\end{equation}
where we use $\zetab_N(\xb)$ to denote the $2N$-dimensional  sinusoidal vector inside the expectation and the frequencies $\{\omegabb_i\}_{i=1}^N$ are sampled i.i.d. from $p(\omegabb)$. 

Applying this to the Gaussian kernel $\exp(-\frac 12 \|\varphib(s, a)-\nub(s')\|^2)$, we have
\begin{equation}
    \begin{aligned}
        \PP^*_\ebm(s'|s, a)&=\frac{\exp(\|\varphib(s, a)\|^2)\exp(-\frac 12\|\varphib(s, a)-\nub(s')\|^2)\exp(\|\nub(s')\|^2)}{Z(s, a)}\\
        &=\Big\langle\underbrace{\frac{\exp(\|\varphib(s, a)\|^2)}{Z(s, a)}\zetab_N(\varphib(s, a))}_{\phib_\ebm(s, a)}, \underbrace{\exp(\|\nub(s')\|)\zetab_N(\nub(s'))}_{\mub_\ebm(s')}\Big\rangle,
    \end{aligned}
\end{equation}
That is, we can obtain linear factorization for transition operators parameterized as factorized EBMs, and the corresponding spectral representation is denoted $\phib^*_\ebm(s, a)$. 

As an ending note, we can continue to demonstrate that the normalizing factor $Z(s, a)$ can also be expressed by $\varphib(s, a)$, such that the EBM spectral representation $\phib_\ebm$ depends exclusively on the factorized energy $\varphib$ and random Fourier frequencies. This can be examined by
\begin{equation}
    \begin{aligned}
        Z(s, a)&=\int_\Scal\Big\langle\exp(\|\varphib(s, a)\|^2)\zetab_N(\varphib(s, a)), \mub_\ebm(s')\Big\rangle\mathrm{d}s'\\
        &=\Big\langle\exp(\|\varphib(s, a)\|^2)\zetab_N(\varphib(s, a)), \underbrace{\int_\Scal\mub_\ebm(s')\mathrm{d}s'}_{\boldsymbol{\ub}}\Big\rangle.\\
    \end{aligned}
\end{equation}
Therefore, we have
\begin{equation}\label{eq:ebm_spectral_representation}
    \begin{aligned}
        \phib_\ebm(s, a)&=\frac{\exp(\|\varphib(s, a)\|^2)\zetab_N(\varphib(s, a))}{\langle\exp(\|\varphib(s, a)\|^2)\zetab_N(\varphib(s, a)), \ub\rangle}\\
        &=\frac{\zetab_N(\varphib(s, a))}{\langle\zetab_N(\varphib(s, a)), \ub\rangle},\quad\text{where }\ub\in\RR^{2N}.
    \end{aligned}
\end{equation}

\paragraph{Remark.} The case of stochastic nonlinear control \citep{zheng2022optimistic,ren2022free,ren2025stochastic,ma2025offline}, whose state-space dynamics follows the form $s'=\varphib(s,a)+v$ where $v$ is a Gaussian noise $\mathcal{N}(0,\sigma^2I)$, can be understood as a special instance of either the energy-based formulation or latent-variable formulation. In this case, the state transition dynamics can be written as $\PP(s'|s, a)\propto\exp\left(-\frac{\|s'-\varphib(s, a)\|^2}{2\sigma^2}\right)$. By applying the same random Fourier features, we can recover the spectral structure analogous to \eqref{eq:ebm_spectral_representation}. Alternatively, the connection to the latent-variable formulation can be seen through the decomposition:
\begin{equation}
    \begin{aligned}
        \PP(s'|s, a)=\frac{1}{\sqrt{2\pi}\sigma}\exp\left(-\frac{\|s'-\varphib(s, a)\|^2}{2\sigma^2}\right)=\inner{p(\cdot|s, a)}{p(s'|\cdot)}_{L_2}, 
    \end{aligned}
\end{equation}
where $p(z|s, a)\propto \exp(-\|z-\varphib(s, a)\|^2/\sigma^2)$ and $p(s'|z)\propto \exp(-\|z-s'\|^2/\sigma^2)$. Our formulations offer greater flexibility by conducting the modeling in the latent space and by generalizing the latent distributions beyond Gaussian distributions, respectively.

\section{Algorithms for Learning Spectral Representations}\label{sec:learning_spectral_representations}
In Section~\ref{sec:spectral_representations}, we studied how spectral representations can be extracted from transition operators with different structures, using either a linear formulation $\phib_\linear$, a latent variable formulation $\phib_\lv$, or an energy-based formulation $\phib_\ebm$. We will discuss the corresponding learning methods for each parametrization from different perspectives.


We begin by estimating the representation $\phib_\linear$ in the context of linear formulation~\eqref{eq:linear_spectral_representation}. Several works on linear MDPs \citep{agarwal2020flambe,uehara2021representation} derive the representation by solving the following maximum log-likelihood estimation (MLE) problem:
\begin{equation}\label{eq:mle}
    \begin{aligned}
        &\max_{\phib, \mub} \ \EE_{(s, a)\sim\rho, s'\sim \PP(\cdot|s, a)}\left[\log \langle\phib(s, a),\mub(s')\rangle\right]\\
        &\ \ \ \text{s.t.}\ \forall(s, a), \ \  \int_\Scal \langle\phib(s, a),\mub(s')\rangle\mathrm{d}s'=1,
    \end{aligned}
\end{equation}
which, however, is notoriously difficult for a large or continuous state space, since enforcing or even just computing the constraint requires an integral over the whole state space. To sidestep this burden, we review tractable alternatives to derive the representations in this section. 

\subsection{Representations from Spectral Contrastive Learning}
With a slight abuse of notation, let $\PP(s, a, s')$ denote the joint distribution of a state-action pair and its successor state, and let $\PP(s, a)$ and $\PP(s')$ be the corresponding marginals. A straightforward approach is to parameterize the conditional density $\PP(s'|s, a)=\varphib_\theta(s ,a)^\top\nub_\theta(s')$ with parameters $\theta$ and minimize the point-wise squared error:
\begin{equation}
    \begin{aligned}
        \min_{\theta}\ \int_{\Scal\times\Acal}\int_\Scal\nbr{\PP(s'|s, a) - \varphib_\theta(s, a)^\top\nub_\theta(s')}^2\mathrm{d}s\mathrm{d}a\mathrm{d}s'. 
    \end{aligned}
\end{equation}
Apparently, the triple integral over spaces makes it an intractable objective. Instead, \speder \citep{ren2023spectral} considers a parameterization modulated by a scaling factor $\PP(s')$:
\begin{equation}\label{eq:speder_identity}
    \PP(s'|s, a)=\varphib_\theta(s, a)^\top\left(\PP(s')\nub_\theta(s')\right).
\end{equation}
As shown later, as long as $\PP(s')$ is amenable to sampling, introducing this factor leads to a tractable optimization objective without altering the final representation. To train $\varphib_\theta$ and $\nub_\theta$, \speder leverages spectral contrastive learning \citep{haochen2021provable,ren2023spectral} to match a rebalanced version of \eqref{eq:speder_identity}:
\begin{equation}\label{eq:spectral_contrastive_loss}
    \begin{aligned}
        \min_{\theta}\ \ell_{\mathrm{SCL}}(\theta)&=\int_{\Scal\times\Acal}\int_{\Scal}\nbr{\frac{\PP(s, a, s')}{\sqrt{\lambda\PP(s, a)\PP(s')}} - \sqrt{\lambda\PP(s, a)\PP(s')}\varphib_\theta(s, a)^\top\nub_\theta(s')}^2\mathrm{d}s\mathrm{d}a\mathrm{d}s'\\
        &=\lambda\EE_{\PP(s, a)\PP(s')}\left[\left(\varphib_\theta(s, a)^\top\nub_\theta(s')\right)^2\right]-2\EE_{\PP(s, a, s')}\left[\varphib_\theta(s, a)^\top\nub_\theta(s')\right]\\
        &\quad\quad\quad+\underbrace{\int_{\Scal\times\Acal}\int_{\Scal}\nbr{\frac{\PP(s, a, s')}{\sqrt{\lambda\PP(s, a)\PP(s')}}}^2\mathrm{d}s\mathrm{d}a\mathrm{d}s'}_{\text{Const}}\\
    \end{aligned}
\end{equation}
This yields a practical contrastive learning-style objective that is compatible with stochastic gradient descent. Specifically, we can approximate the expectations in the objective by drawing samples from the empirical distributions (i.e., the replay buffer). For each state-action pair $(s, a)$ drawn uniformly from the buffer, the ground-truth successor state serves as the positive sample $s'_+$, and states randomly sampled from the buffer act as negative samples $s'_-$, thereby implementing the spectral contrastive loss as:
\begin{equation}\label{eq:empirical_spectral_contrastive_loss}
    \begin{aligned}
        \min_{\theta}\ \ell_{\mathrm{SCL}}(\theta)= \frac{\lambda}{N}\sum_{n=1}^N\left(\varphib_\theta(s_n, a_n)^\top\nub_\theta(s_{n,-}')\right)^2-\frac{2}{N}\sum_{n=1}^N\left(\varphib_\theta(s_n, a_n)^\top\nub_\theta(s_{n,+}')\right) + \text{Const}.
    \end{aligned}
\end{equation}
Finally, the spectral representation can be recovered by $\phib_\linear\approx\varphib_{\theta^*}$. It must be noted that since in practice we must settle for a finite-dimensional representation $\varphib_{\theta^*}$, such representations have limited expressive power for both transition probabilities and $Q$-value functions.

\paragraph{Remark (Singular Value Decomposition):} Another alternative interpretation of \eqref{eq:spectral_contrastive_loss} is that it is equivalent to the SVD of a scaled transition operator. Define $L_2(\Scal\times\Acal)$ and $L_2(\Scal)$ be the spaces of square summable functions over $\Scal\times\Acal$ and $\Scal$ respectively, we consider the scaled transition kernel $\widetilde{\Tcal}: L_2(\Scal\times\Acal)\to L_2(\Scal)$ which satisfies $(\widetilde{T}f)(s')=\int\frac{\PP(s, a, s')}{\sqrt{\PP(s, a)\PP(s')}}f(s, a)\mathrm{d}(s, a)$. Its SVD tries to find orthogonal eigen-functions $\varphib=\{\varphi_i\}_{i=1}^I$ to capture the principal components:
\begin{equation}\label{eq:svd}
    \begin{aligned}
        &\max_{\varphi_i^\top \varphi_j=\mathbf{1}_{i=j}}\ \sum_{i=1}^I\Bigg\|\int\frac{\PP(s, a, s')}{\sqrt{\PP(s, a)\PP(s')}}\varphi_i(s, a)\mathrm{d}(s, a)\Bigg\|_2^2\\
        &=\max_{\varphi_i^\top \varphi_j=\mathbf{1}_{i=j}}\ \max_{\nu_1, \ldots, \nu_I}\ \sum_{i=1}^I\left(2\iint \frac{\PP(s, a, s')}{\sqrt{\PP(s, a)\PP(s')}}\varphi_i(s, a)\nu_i(s, a)\mathrm d{(s, a)}\mathrm{d}s' - \int \nu_i(s')^2\mathrm{d}s'\right)\\
        &=\max_{\EE_{\PP(s, a)}[\varphi_i' \varphi'_j]=\mathbf{1}_{i=j}}\ \max_{\nu'_1, \ldots, \nu'_I}\ \sum_{i=1}^I\left(2\iint \PP(s, a, s')\varphi'_i(s, a)\nu'_i(s, a)\mathrm d{(s, a)}\mathrm{d}s' - \int \PP(s')\nu'_i(s')^2\mathrm{d}s'\right)\\
        &=\max_{\EE_{\PP(s, a)}[\varphi_i' \varphi'_j]=\mathbf{1}_{i=j}}\ \max_{\nu'_1, \ldots, \nu'_I}\ 2\EE_{\PP(s, a, s')}\left[\sum_{i=1}^I \varphi_i'(s, a)\nu_i'(s')\right]-\EE_{\PP(s, a)\PP(s')}\left[\left(\sum_{i=1}^I\varphi_i'(s, a)\nu_i'(s')\right)^2\right]\\
    \end{aligned}
\end{equation}
where the first equality follows from the Fenchel duality of $\|\cdot\|_2^2$, the second equality comes from the change of variable $\varphi_i'(s, a)=\frac{\varphi_i(s, a)}{\sqrt{\PP(s, a)}}$ and $\nu_i'(s')=\frac{\nu_i(s')}{\sqrt{\PP(s')}}$, and the last equality is due to the orthogonality constraint $\EE[\varphi_i(s, a)\varphi_j(s, a)]=\mathbf{1}_{i=j}$:
\begin{equation}
    \begin{aligned}
        \EE_{\PP(s, a)\PP(s')}\left[\left(\sum_{i=1}^I\varphi_i'(s, a)\nu_i'(s')\right)^2\right]&=\EE_{\PP(s, a)\PP(s')}\left[\sum_{i=1}^I\left(\varphi_i'(s, a)\nu_i'(s')\right)^2\right]\\
        &=\EE_{\PP(s')}\left[\sum_{i=1}^I\left(\nu_i'(s')\EE_{\PP(s, a)}[\varphi_i'(s, a)\varphi_i'(s, a)]\nu_i'(s')\right)\right]\\
        &=\EE_{\PP(s')}\left[\sum_{i=1}^I\nu_i'(s')^2\right].
    \end{aligned}
\end{equation}
Recognizing the summation in \eqref{eq:svd} as the inner product between the representations, the spectral contrastive loss in \eqref{eq:spectral_contrastive_loss} is precisely the variational objective for an SVD of the scaled transition kernel, without the orthonormal constraint. Therefore, spectral contrastive loss~\eqref{eq:spectral_contrastive_loss} recovers the same subspace of SVD. 

\subsection{Representations from Variational Learning}
When the transition operator possesses the latent variable structure as described in \secref{sec:lv}, \lvrep \citep{ren2023latent} employs variational learning to obtain a tractable surrogate objective of the original MLE \eqref{eq:mle}. Specifically, suppose the transition operator is parameterized as $\PP(s' |s, a)=\int \varphi_\theta(z|s, a)\nu_\theta(s'|z)\mathrm{d}z$ for some valid probability measure $\varphi_\theta:\Scal\times\Acal\to\Delta(\Zcal)$, then it follows from Jensen's inequality that
\begin{equation}
    \begin{aligned}
        \log \PP(s'|s, a)&=\log \int \varphi_\theta(z|s, a)\nu_\theta(s'|z)\mathrm{d}z=\log \int q_\theta(z|s, a, s')\frac{\varphi_\theta(z|s, a)\nu_\theta(s'|z)}{q_\theta(z|s, a, s')}\mathrm{d}z\\
        &\geq \underbrace{\EE_{z\sim q_\theta(\cdot|s, a, s')}\left[\log \nu_\theta(s'|z)\right] - D_{\mathrm{KL}}(q_\theta(\cdot|s, a, s')\| \varphi_\theta(\cdot|s, a))}_{\ell_{\mathrm{ELBO}}(\theta)},
    \end{aligned}
\end{equation}
where $q_\theta(\cdot|s, a, s')$ is the variational distribution also parameterized by $\theta$, and the RHS of the last inequality is also known as the evidence lower bound (ELBO) \citep{kingma2013auto}. We can also introduce a hyper-parameter $\beta\geq 0$ and optimize the $\beta$-VAE objective \citep{higgins2017beta} to balance the KL regularization term and the reconstruction term:
\begin{equation}\label{eq:bvae_loss}
    \begin{aligned}
        \ell_{\mathrm{ELBO}, \beta}(\theta)=\EE_{z\sim q_\theta(\cdot|s, a, s')}\left[\log \nu_\theta(s'|z)\right] - \beta D_{\mathrm{KL}}(q_\theta(\cdot|s, a, s')\| \varphi_\theta(\cdot|s, a)).
    \end{aligned}
\end{equation}
Similarly, with the optimal parameter $\theta^*$ obtained by maximizing $\ell_{\mathrm{ELBO}, \beta}$, the spectral representation $\phib_\lv$ can be set as $\varphi_{\theta^*}(\cdot|s, a)$ based on \eqref{eq:lvrep}.

\subsection{Representations from Score Matching}\label{sec:score_matching}

In section \ref{sec:ebm}, we described how to construct spectral representations $\phib_\ebm$ from a factorized energy function $\varphib(s, a)^\top\nub(s')$. A key challenge with this is that estimating the energy function of EBMs with MLE requires computing the normalization factor, which is often intractable \citep{lecun2006tutorial,song2021train}. Diffusion models \citep{ho2020denoising,songscore} offer a promising solution for sampling from EBMs by estimating the \textit{score functions} instead of the energy functions, thereby bypassing the need for this normalization factor. Inspired by this, \diffsr \citep{shribak2024diffusion} derives a tractable, score-based optimization objective for learning $\phib(s, a)$. 

For a given state-action pair $(s, a)$, \diffsr perturbs the samples from the ground-truth transition $s'\sim\PP(s'|s, a)$ using a Gaussian kernel with a noise level chosen from a pre-defined schedule $\{\beta_m\}_{m=1}^M$. For a given noise level $\beta$, the perturbation kernel $\PP(\stil'|s'; \beta)$ and the perturbed transition $\PP(\stil'|s, a; \beta)$ are defined as:
\begin{equation}\label{eq:perturbed_ebm}
    \begin{aligned}
    \PP(\stil'|s'; \beta)&=\Ncal(\sqrt{1-\beta}s', \beta I), \\
    \PP(\stil'|s, a; \beta) &=\int_\Scal\PP(s'|s, a)\PP(\stil'|s'; \beta)\mathrm{d}s'. 
    \end{aligned}
\end{equation}
where $\PP(\stil'|s, a; \beta)\to\PP(\stil'|s, a)$ as $\beta\to 0$. Following the EBM formulation, \diffsr parameterizes this perturbed transition with $\theta$:
\begin{equation}
    \begin{aligned}
        \PP(\stil|s, a; \beta) \propto \exp(\varphib_\theta(s, a)^\top\nub_\theta(\stil'; \beta)), 
    \end{aligned}
\end{equation}
where $\varphib_\theta: \Scal\times\Acal\to\RR^d$ is shared across all noise levels. To learn $\varphib_\theta$, \diffsr matches the score function of the parameterized distribution $\nabla_{\stil'}\log \PP(\stil'|s, a; \beta)\approx\varphib_\theta(s,a)^\top\nabla_{\stil'}\nub_\theta(\stil'; \beta)$ with that of the ground-truth corrupted transition, which yields the following score matching objective:
\begin{equation}
    \begin{aligned}
        \ell_{\mathrm{SM}}(\theta)&=\EE_{(s, a)\sim\rho, \stil'\sim\PP(\cdot|s, a; \beta)}\left[\nbr{\varphib_\theta(s, a)^\top\nabla_{\stil'}\nub_\theta(\stil'; \beta) - \nabla_{\stil'}\log \PP(\stil'|s, a; \beta)}^2\right],\\
    \end{aligned}
\end{equation}
However, the ground-truth score remains intractable, as it involves integrating over all possible next states $s'$. To address this, \diffsr uses the conditional score matching objective:
\begin{equation}\label{eq:conditional_score_matching}
    \begin{aligned}
        \ell_{\mathrm{CSM}}(\theta)=\EE_{(s, a)\sim\rho, s'\sim \PP(\cdot|s, a), \stil'\sim\PP(\cdot|s'; \beta)}\left[\nbr{\varphib_\theta(s, a)^\top\nabla_{\stil'}\nub_\theta(\stil'; \beta) - \nabla_{\stil'}\log \PP(\stil'|s'; \beta)}^2\right],
    \end{aligned}
\end{equation}
which shares the same optimal solution to $\ell_{\mathrm{SM}}$, as proved in the Appendix A of \citet{shribak2024diffusion}. Note that the regression target has analytical forms when the perturbation kernel is known, and in our case, $\nabla_{\stil'}\log\PP(\stil'|s';\beta)=-(\stil'-\sqrt{1-\beta}s')/{\beta}$. In practice, \diffsr directly parameterizes $\nabla_{\stil'}\nub_\theta(\stil'; \beta)$ as a neural network $\boldsymbol{\kappa}_\theta:\RR^{\dim(\Scal)}\times\RR\to\RR^{d\times\dim(\Scal)}$ and thereby avoids the need to compute second-order gradients. 

After training is completed, spectral representations $\phib_\ebm$ can be obtained according to $\phib_\ebm\approx\frac{\zetab_N(\varphib_{\theta^*}(s, a))}{\langle\zetab_N(\varphib_{\theta^*}(s, a)), \ub\rangle}$ \eqref{eq:ebm_spectral_representation}, where $\ub$ will be specified later in section \ref{sec:rl_with_spectral_rep}. 

\subsection{Representations from Noise Contrastive Estimation}

Another effective approach to learning the energy function is noise contrastive estimation (NCE) \citep{ma2018noise,gutmann2010noise,gutmann2012noise}. Specifically, let $\widetilde{\PP}(s')$ be an arbitrary noise distribution, and we model the target density as: 
\begin{equation}\label{eq:ebm_parameterization}
    \begin{aligned}
        \PP(s'|s, a)\propto \widetilde{\PP}(s')\exp\left(\varphib_\theta(s, a)^\top\nub_\theta(s')\right).
    \end{aligned}
\end{equation}
For each state-action pair $(s_n, a_n)$, we draw one positive sample $s'_{n, 0}\sim \PP(s'|s, a)$ and $K$ negative samples $\{s'_{n, k}\}_{k=1}^K$ i.i.d. from the noise distribution $\widetilde{\PP}(s')$. The \emph{ranking-based NCE} tries to identify the positive sample among the negatives:
\begin{equation}
    \begin{aligned}
        \ell_{\mathrm{R-NCE}}(\theta)&=\frac 1N\sum_{n=1}^N\frac{\exp(\varphib_\theta(s_n, a_n)^\top\nub_\theta(s'_{n,0}))}{\sum_{k=1}^K\exp(\varphib_\theta(s_n, a_n)^\top\nub_\theta(s'_{n,k}))}
    \end{aligned}
\end{equation}
The optimal parameter $\theta^*=\argmin_{\theta}\ell_{\mathrm{R-NCE}}(\theta)$ provides an estimate of the energy function \eqref{eq:ebm_parameterization}, combined with the negative distribution. Note that the noise distribution $ \widetilde{\PP}(s')$ in \eqref{eq:ebm_parameterization} does not affect the extracted spectral representation $\phib_\ebm$, as it depends only on $s'$ and can be absorbed into the other spectral component, $\mub_\ebm$. 

However, when the negative sample distribution $\widetilde{\PP}(s')$ is trivial or disjoint from the true conditional density $\PP(s'|s, a)$, NCE may yield poor representation since classifying between the positives and the negatives is too easy \citep{robinson2020contrastive,rhodes2020telescoping}. While existing methods often address this by using more negative samples \citep{radford2021learning,chen2022we} or designing more sophisticated noise distributions \citep{finn2016connection}, we propose a simple solution inspired by diffusion models. 

Specifically, we conduct NCE across different perturbed negative distributions $\widetilde{\PP}(\stil';\beta)$ and perturbed EBMs $\PP(\stil'|s, a; \beta)$, each of which is associated with a noise level $\beta$ from the pre-defined schedule $\{\beta_m\}_{m=1}^M$:
\begin{equation}
    \begin{aligned}
        \widetilde{\PP}(\stil';\beta)&=\int_\Scal \PP(\stil'|s';\beta)\widetilde{\PP}(s')\mathrm{d}s',\\
        \PP(\stil'|s, a; \beta) &=\widetilde{\PP}(\stil';\beta)\int_\Scal\PP(s'|s, a)\PP(\stil'|s'; \beta)\mathrm{d}s'\\
        &\propto \widetilde{\PP}(\stil';\beta)\exp(\varphib_\theta(s, a)^\top\nub_\theta(\stil'; \beta)).
    \end{aligned}
\end{equation}
where the perturbation kernel is $\PP(\stil'|s';\beta)=\Ncal(\sqrt{1-\beta}s', \beta I)$. Similar to the score matching case in section \ref{sec:score_matching}, the $\varphib_\theta$ is shared across different noise levels. As illustrated in Figure \ref{fig:nce_perturbed_distributions}, when $\beta\to0$, the perturbed distribution falls back to the original distributions; while for larger $\beta$, both the positive and negative distributions become smoother and increasingly overlap. This creates a more challenging classification task at higher noise levels, which acts as a powerful regularizer for the contrastive learning process. The final optimization objective, which we term \emph{ranking-based perturbed NCE (RP-NCE)}, is an average of the NCE losses across all noise levels:
\begin{equation}\label{eq:rp_nce}
    \begin{aligned}
        \ell_{\mathrm{RP-NCE}}(\theta)&=\frac 1{MN}\sum_{m=1}^M\sum_{n=1}^N\log\frac{\exp(\varphib_\theta(s_n,a_n)^\top\nub_\theta(\stil'_{n,0};\beta_m))}{\sum_{k=1}^K\exp(\varphib_\theta(s_n,a_n)^\top\nub_\theta(\stil'_{n,k};\beta_m))},
    \end{aligned}
\end{equation}
where $\stil'_{n,0}$ and $\stil'_{n,k}$ are sampled by perturbing the positive sample $s'_{n,0}$ and negative samples $s'_{n,k}$ with the Gaussian distribution $\PP(\cdot|s';\beta_k)$, respectively.  

\begin{figure}[t]
    \centering
    \includegraphics[width=1.0\linewidth]{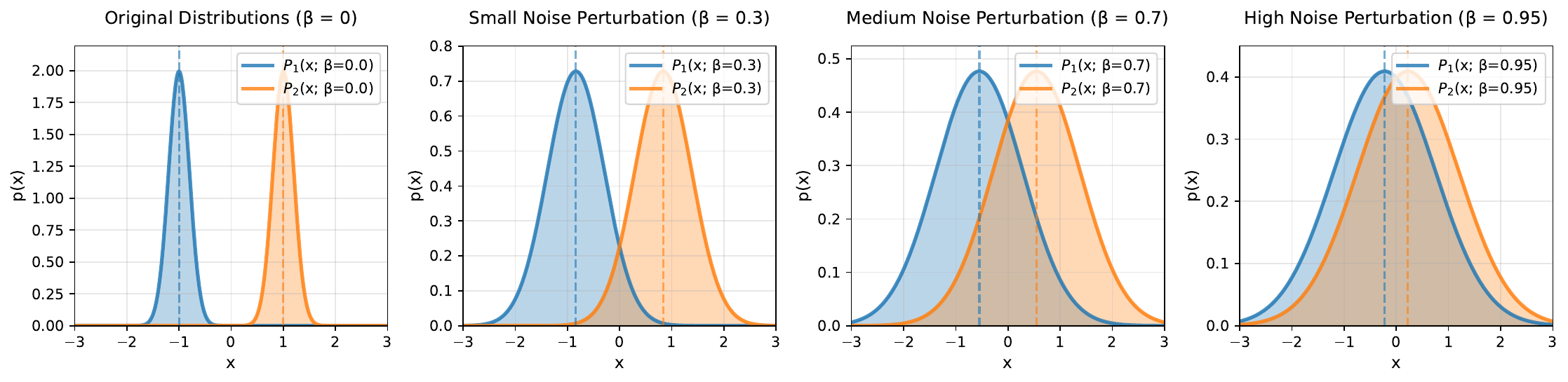}
    \caption{Illustrations of two nearly disjoint Gaussian distributions and their perturbed versions under various configurations of $\beta$. As $\beta$ progresses from $0$ to $1$, they overlap with each other and it becomes increasingly difficult for NCE to tell the difference between these distributions.}
    \label{fig:nce_perturbed_distributions}
\end{figure}

\paragraph{Remark (NCE for Linear Spectral Representations)} \ctrl \citep{zhang2022making} leverages the NCE framework to derive linear spectral representations. It formulates the target density as: 
\begin{equation}
\begin{aligned}\label{eq:ncel}
    &\quad\PP(s'|s, a) \propto \widetilde{\PP}(s')\varphib_\theta(s, a)^\top\nub_\theta(s, a)\\
    &\st\quad \varphib_\theta(s, a)^\top{\nub_\theta(s')} \geq 0, \quad \forall (s, a, s')
\end{aligned}
\end{equation}
The non-negativity constraint is necessary to ensure that $\PP(s'|s, a)$ is a valid probability measure. The corresponding ranking-based perturbed NCE objective is:
\begin{equation}
    \begin{aligned}\label{eq:rp_ncel}
        \ell_{\mathrm{RP-NCEL}}(\theta)&=\frac 1{MN}\sum_{m=1}^M\sum_{n=1}^N\log\frac{\varphib_\theta(s_n,a_n)^\top\nub_\theta(\stil'_{n,0};\beta_m)}{\sum_{k=1}^K\varphib_\theta(s_n,a_n)^\top\nub_\theta(\stil'_{n,k};\beta_m)}.
    \end{aligned}
\end{equation}
For spectral representations, \ctrl uses $\phib_\linear\approx\varphib_{\theta^*}$. In practice, to enforce the non-negativity constraint in \eqref{eq:ncel}, we can activate the inner product in \eqref{eq:rp_ncel} with \texttt{SoftPlus}, a smoothed version of \texttt{ReLU}:
\begin{equation}
    \texttt{SoftPlus}(x) = \log(1+e^x).
\end{equation}
Doing so will inevitably introduce non-linearity; however, we can encourage the model to operate in the linear regime of the \texttt{SoftPlus} function by adding a penalty term to push the inner product $\varphib_\theta(\cdot, \cdot)^\top\nub_\theta(\cdot)$ toward large positive values, such that $\texttt{SoftPlus}(x)\approx x$ and the linearity is preserved.

\section{Reinforcement Learning with Spectral Representations}\label{sec:rl_with_spectral_rep}

\begin{algorithm}[t]
\caption{Online Reinforcement Learning with Spectral Representations}
\label{alg}
\textbf{Initialize}: replay buffer $\Dcal=\emptyset$, policy $\pi_\psi$, representation networks $\varphib_\theta$ and $\nub_\theta$, reward network $r_{\theta, \xi_r}$, and $Q$-value functions $Q_{\theta, \xi_1}$ and $Q_{\theta, \xi_2}$ according to the linear formulation \eqref{eq:q_linear}, latent variable formulation \eqref{eq:q_lv} or energy-based formulation \eqref{eq:q_ebm}.

\begin{algorithmic}[1]
\FOR{$t=1, 2, \cdots, T_{\text{total\_steps}}$}
    \STATE $a_t\sim\pi(\cdot|s_t)$
    \STATE $r_t=r(s_t, a_t), s'_t\sim\PP(\cdot|s_t, a_t)$
    \STATE \emph{(Optional)} Compute bonus $b(s_t, a_t)$ 
    \STATE $\Dcal\leftarrow \Dcal\cup\{(s_t, a_t, r_t, s_t')\}$
    \STATE Update $\varphib_\theta$, $\nub_\theta$, and $r_{\theta, \xi_r}$ with spectral contrastive loss \eqref{eq:spectral_contrastive_loss}, variational learning \eqref{eq:bvae_loss}, score matching \eqref{eq:conditional_score_matching}, noise contrastive estimation \eqref{eq:rp_nce} or its linear form \eqref{eq:rp_ncel}, combined with the reward prediction loss \eqref{eq:reward_prediction}
    \STATE Update the value networks $Q_{\theta,\xi_1}$ and $Q_{\theta, \xi_2}$ with \eqref{eq:critic_loss}
    \STATE Update the policy $\pi_\psi$ with \eqref{eq:actor_loss}
\ENDFOR

\textbf{Return} $\pi_\psi$
\end{algorithmic}
\end{algorithm}



As established in Lemma \ref{lemma:q_sufficiency}, after we obtain the spectral representation, $Q$-value functions of any policy $\pi$ can be expressed with one of the spectral representations: either $\phib_\linear$, $\phib_\lv$, or $\phib_\ebm$. This connection inspires us to parameterize the $Q$ functions on top of these representations, as we will detail below. 

For the linear spectral representation where $\phib_\linear\approx\varphib_\theta$, $Q$ functions can be parameterized as a linear function of the features with a weight vector $\xi\in\RR^d$, 
\begin{equation}\label{eq:q_linear}
    Q_{\theta, \xi}( s, a) = \varphib_\theta^\top\xi.
\end{equation}
For the latent variable representation $\phib_\lv$, we can traverse the latent variables if $\Zcal$ is discrete and finite or we can approximate $Q$-functions with Monte-Carlo estimations if $\Zcal$ is continuous:
\begin{equation}\label{eq:q_lv}
    \begin{aligned}
        Q_{\theta, \xi}(s, a)=\inner{\varphi_\theta(\cdot|s, a)}{\xi(\cdot)}_{L_2}=\left\{
        \begin{aligned}
            &\sum_{i=1}^{|\Zcal|}\varphi_\theta(z_i|s, a)\xi(z_i)\quad\quad\quad\quad\quad\quad\quad\ \ \ \text{(discrete)}\\
            &\int_\Zcal \varphi_\theta(z|s, a)\xi(z)\mathrm{d}z\approx \frac 1L\sum_{l=1}^L\xi(z_l)\quad\text{(continuous)}\\
        \end{aligned}\right.
    \end{aligned}
\end{equation}
where $\xi: \Zcal\to\RR$ is a function mapping the latent variable to a scalar value, typically parameterized by a neural network; $\{z_l\}_{l=1}^L$ are Monte-Carlo samples from $\varphi_\theta(\cdot|s, a)$.

For the energy-based representations $\phib_\ebm$, the parameterization is more complex due to its non-linear dependence on $\varphib_\theta$, as defined in \eqref{eq:ebm_spectral_representation}. This representation involves random Fourier features, $\zetab_N(\varphib(s, a))$, which we approximate using sinusoidal activations, i.e., a concatenation of $\sin(W_1^\top\varphib_\theta(s, a))$ and $\cos(W_1^\top\varphib_\theta(s, a))$. To account for the additional non-linearity from the term $\frac{1}{\langle\zetab_N(\varphib(s, a)), \ub\rangle}$, we employ an additional network layer. This results in the following parameterization for the Q-function:
\begin{equation}\label{eq:q_ebm}
    \begin{aligned}
        Q_{\theta, \xi}=\texttt{activ}(W_2\ [\cos(W_1\varphib_\theta(s, a)), \sin(W_1\varphib_\theta(s, a))]^\top)^\top\eta,
    \end{aligned}
\end{equation}
where $\texttt{activ}(\cdot)$ is some nonlinear activation function and $\xi=(W_1, W_2, \eta)$ are learnable parameters optimized by the critic loss. 

To also enforce the linear structure in the reward function \eqref{eq:spectral_mdp}, we parameterize a reward function $r_{\theta, \xi_r}$ using the same structure as $Q$-value functions, and include a reward prediction objective for the representation networks:
\begin{equation}\label{eq:reward_prediction}
    \begin{aligned}
        \ell_{\text{reward}}(\theta, \xi_r)&=\EE_{(s, a, s', r)\in \Dcal}\left[\left(r_{\theta, \xi_r}(s, a) - r\right)^2\right].
    \end{aligned}
\end{equation}

Following standard practice, we apply the \emph{double Q-network trick} \citep{fujimoto2018addressing} to stabilize learning and combat over-estimation bias. Specifically, we maintain two independent set of parameters $\xi_1,\xi_2$ and their exponential moving average (EMA) version $\bar{\theta}_1, \bar{\theta}_2$, and update the $Q$-value functions using the standard TD learning objective:
\begin{equation}\label{eq:critic_loss}
    \ell_{\text{critic}}(\xi_1, \xi_2)=\EE_{(s, a, r, s')\sim\Dcal}\left[\sum_{i\in\{1,2\}}\left(r+\gamma\EE_{a'\sim\pi}\left[\min
    _{j\in\{1,2\}}Q_{\theta,\bar{\xi}_j}(s',a')\right] - Q_{\theta, \xi_i}\right)^2\right].
\end{equation}
The policy is then updated by maximizing the expected Q-value:
\begin{equation}\label{eq:actor_loss}
    \begin{aligned}
        \ell_{\text{policy}}(\pi)=\EE_{s\sim\Dcal, a\sim\pi(\cdot|s)}\left[\min_{i\in\{1,2\}}Q_{\theta, \xi_j}(s, a)\right].
    \end{aligned}
\end{equation}
Algorithm \ref{alg} presents the pseudo-code of online RL algorithms equipped with spectral representations, where the data collection, policy optimization, and representation learning are performed simultaneously. Note that the training of critics and representation networks is decoupled, \ie, the TD learning objective $\ell_{\text{critic}}$ is not used to train $\varphib_\theta$ and $\nub_\theta$. In fact, the framework of spectral representations can be readily integrated into a wide range of reinforcement learning algorithms, including those designed for offline scenarios and visual-input tasks. This showcases the wide applicability of spectral representations. 

\subsection*{Theoretical Analysis}\label{sec:theoretical_analysis}

The optimistic exploration with spectral representation has been justified rigorously in~\citep{ren2023latent,ren2023spectral}, which we briefly introduce here. 
We first make the following assumptions about the candidate class and normalization conditions, which are common among similar analysis.

\begin{assumption}
    Let $|\Pcal|<\infty$ and $\phib\in\Pcal, \mub\in\Pcal$. For $\forall(s, a)\in\Scal\times\Acal$ and $\forall s'\in\Scal$, $\phib(s, a)\in\Hcal_k$ and $\mub(s')\in\Hcal_k$ for some RKHS $\Hcal_k$ with kernel $k$. 
\end{assumption}
Unlike the linear MDP case \citep{jin2020provably,uehara2021representation}, we consider representations in the RKHS $\Hcal_k$, which includes all formulations in \secref{sec:spectral_representations} and provides an abstraction. For example, the energy-based formulation can be recognized as specifying $k$ as the Gaussian kernel \eqref{eq:ebm2}, while the linear formulation can be recovered using a linear kernel. 
\begin{assumption}{(Normalizing Conditions)}\label{asmp:normalizing_conditions}
    Let kernel $k$ be defined on a compact metric space $\Zcal$ with the Lebesgue measure $\mu$ if $\Zcal$ is continuous, and $\int_\Zcal k(z, z)\mathrm{d}z\leq 1$. For $\forall (\phib,\mub)\in\Pcal,\|\phib(s, a)\|_{\Hcal_k}\leq 1$. Besides, $\forall g: \Scal\to\RR$ such that $\|g\|_\infty\leq 1$, we have $\|\int_\Scal \mub(s')g(s')\mathrm{d}s'\|\leq C$.
\end{assumption}
\begin{assumption}{(Eigenvalue Decay Conditions)}\label{asmp:eigendecay}
    For the reproducing kernel, we assume $\mu_i$, the $i$-th eigenvalue of the operator $T_k: L_2(\mu)\rightarrow L_2(\mu)$, $T_kf=\int_\Zcal f(z')k(z, z')\mathrm{d}\mu(z')$, satisfies one of the following conditions:
    \begin{itemize}
        \item $\beta$-finite spectrum: $\mu_i=0$ for $\forall i>\beta$ where $\beta$ is a positive integer;
        \item $\beta$-polynomial decay: $\mu_i\leq C_0i^{-\beta}$, where $C_0$ is an absolute constant and $\beta\geq 1$;
        \item $\beta$-exponential decay: $\mu_i\leq C_1\exp(-C_2i^\beta)$, where $C_1,C_2$ are absolute constants and $\beta>0$.
    \end{itemize}
\end{assumption}
Most common kernels satisfy the above decay conditions. For example, the linear and polynomial kernels satisfy the $\beta$-finite spectrum condition, while the Gaussian kernel satisfies $\beta$-exponential decay. Finally, we present the sample complexity bound for online reinforcement learning. 
\begin{theorem}{(PAC Guarantee for Online Reinforcement Learning)}\label{thm:main}
    Assume the reproducing kernel $k$ satisfies the eigenvalue decay conditions in \asmpref{asmp:eigendecay}. With a proper choice of exploration bonus function $\hat{b}_n(s, a)$, the optimistic version of our algorithm \ref{alg_optimistic} can yield an $\epsilon$-optimal policy with probability at least $1-\delta$ after interacting with the environment for $N=\text{poly}(C, |\Acal|, (1-\gamma)^{-1}, \epsilon, \log(|\Pcal|/\delta))$ episodes. 
\end{theorem}
This theorem and its proof largely follow Theorem 4 in \citet{ren2023latent}, except that the bonus functions $\hat{b}_n$ are defined with the RKHS inner product $\inner{\cdot}{\cdot}_{\Hcal_k}$, rather than $\inner{\cdot}{\cdot}_{L_2(\mu)}$. We refer interested readers to Appendix E of \citet{ren2023latent} for a detailed proof.

\section{Spectral Representations in Partially Observable MDPs}\label{sec:pomdps}
In this section, we demonstrate that the framework of spectral representations can be seamlessly extended to partially observable MDPs (POMDPs), leading to the first practical algorithm that can provably solves a subclass of POMDPs. 
Following \citet{efroni2022provable}, a POMDP is defined by the tuple $\langle \Scal, \Acal, \Ocal, \PP, \OO, r, \gamma, d_0\rangle$, where $\Scal, \Acal,$ and $\Ocal$ are the state, action, and observation spaces, $r: \Ocal\times\Acal \to \RR$ is the reward function, $\PP: \Scal\times\Acal\to\Delta(\Scal)$ is the transition probability over the latent states, and $\OO: \Scal\to\Delta(\Ocal)$ is the emission probability that governs the observation an agent will receive given the system state.

By introducing the belief function $b(\cdot):(\Ocal\times\Acal)^t\times\Ocal\to \Delta(\Scal)$ that maps the observation and action history to a distribution over the system states, POMDPs can be converted into equivalent belief state MDPs $\Mcal_b=\langle\Bcal, \Acal, \PP_b, r_b, \gamma, d_{b,0}\rangle$. Here, $\Bcal\subseteq \Delta(\Scal)$ is the space of beliefs, $d_{b,0}(\cdot)=\iint d_0(s_0)\OO(o_0|s_0)b(\cdot|o_0)\mathrm{d}s_0\mathrm{d}o_0$, and
\begin{equation}
    \begin{aligned}
        &\PP_b(b_{t+1}|b_t,a_t)=\\&\quad\iiint \mathbf{1}_{b_{t+1}=b(o_{1:t+1},a_t)}\OO(o_{t+1}|s_{t+1})\PP(s_{t+1}|s_{t}, a_t)b_t(s_t|o_{1:t}, a_{t-1})\mathrm{d}s_t\mathrm{d}s_{t+1}\mathrm{d}o_{t+1}.
    \end{aligned}
\end{equation}
The value functions, $V^\pi(b_t)$ and $Q^\pi(b_t, a_t)$, can then be defined over these belief states. However, the belief state MDPs equivalence does not induce a practical algorithm,
since the belief states are not directly observed and are dependent on the entire history. 
Consequently, to reduce the statistical complexity of learning, we consider a special class of POMDPs with the following structure.
\begin{definition}{($L$-decodability \citep{efroni2022provable})}
    Define
    \begin{equation}
        \begin{aligned}
            x_t&\in\Xcal:= (\Ocal\times\Acal)^{L-1}\times\Ocal, \\
            x_t&= (o_{t-L+1}, a_{t-L+1}, \ldots, o_t),
        \end{aligned}
    \end{equation}
    a POMDP is $L$-decodable if there exists a decoder $p: \Xcal\to\Delta(\Scal)$ such that $p(x_t)=s_t$. 
\end{definition}
Intuitively, $L$-decodability implies that the $L$-step history window serves as a sufficient statistic for the state $s_t$. We note that the $L$-decodability is a plausible assumption that commonly holds in practical decision-making scenarios. For example, in visual RL, although a single frame may be insufficient to capture the full system state (e.g., velocity of the robot), stacking consecutive frames effectively compensates for this partial observability and has become the de facto practice in modern visual RL algorithms \citep{mnih2013playing,yarats2021mastering}.

Most importantly, the $L$-decodability assumption eliminates the dependence on earlier trajectory history in transition probabilities, allowing us to define value functions directly over this windowed history, $Q^\pi(x_t, a_t)$, which satisfy the $L$-step Bellman equation:
\begin{equation}
    \begin{aligned}
        Q^\pi(x_t, a_t) = \EE_{x_{t+1:t+L}\sim\PP^\pi(\cdot|x_t, a_t)}\left[\sum_{i=0}^{L-1}\gamma^ir_{t+i}+\gamma^{L}V^\pi(x_{t+L})\right],
    \end{aligned}
\end{equation}
where $V^\pi(x_{t+L})$ is independent of $(x_t, a_t)$ thanks to the $L$-decodability property. To further eliminate the dependence on $x_t$ in $x_{t+1:t+L-1}$, we employ the \emph{moment matching policy} \citep{efroni2022provable} $\tilde{\pi}$, which conditions solely on $(x_t, a_t)$ but generates the same $L$-step observation distribution as $\pi$ (proof provided in \citet{zhang2023provable}, Appendix C). In this way, if we consider the spectral decomposition of the $L$-step transition and reward:
\begin{equation}
    \begin{aligned}
        \PP^\pi_L(x_{t+L}|x_t, a_t) &= \inner{\phib(x_t, a_t)}{\mub^{\tilde{\pi}}(x_{t+L})}_\Hcal,\\
        r^\pi_L(x_t, a_t)&=\sum_{i=0}^{L-1}\gamma^ir_{t+i}=\inner{\phib(x_t, a_t)}{\thetab^{\tilde{\pi}}_r}_\Hcal,\\
    \end{aligned}
\end{equation}
then it follows that $Q^\pi(x_t, a_t)=\inner{\phib(x_t,a_t)}{\thetab_r^{\tilde{\pi}}+\gamma^L\mub^{\tilde{\pi}}(x_{t+L})}$, making $\phib(x_t, a_t)$ the spectral representation that can sufficiently express the $Q$-value function. Consequently, all three formulations in \secref{sec:linear_nonlinear_spectral_representation} and the representation learning methods in \secref{sec:learning_spectral_representations} extend naturally to such $L$-step transition and rewards. For example, noise contrastive estimation now parameterizes the representation networks $\phib(o_{t-L+1:t}, a_t)$ and $\nub(o_{t+1:t+L})$ using $L$ consecutive frames as input and for contrastive learning. For prediction-based approaches such as variational learning, since predicting all $L$ frames may be computationally expensive, we can also work with predicting every single frame as an approximation. Finally, the RL algorithmic implementations and theoretical characterizations in \secref{sec:rl_with_spectral_rep} also carry over to this setting, yielding a family of practical and theoretically grounded algorithms for solving POMDPs.

\section{Empirical Evaluations}\label{sec:experiments}

\begin{figure}[t]
    \centering
    \includegraphics[width=1.0\linewidth]{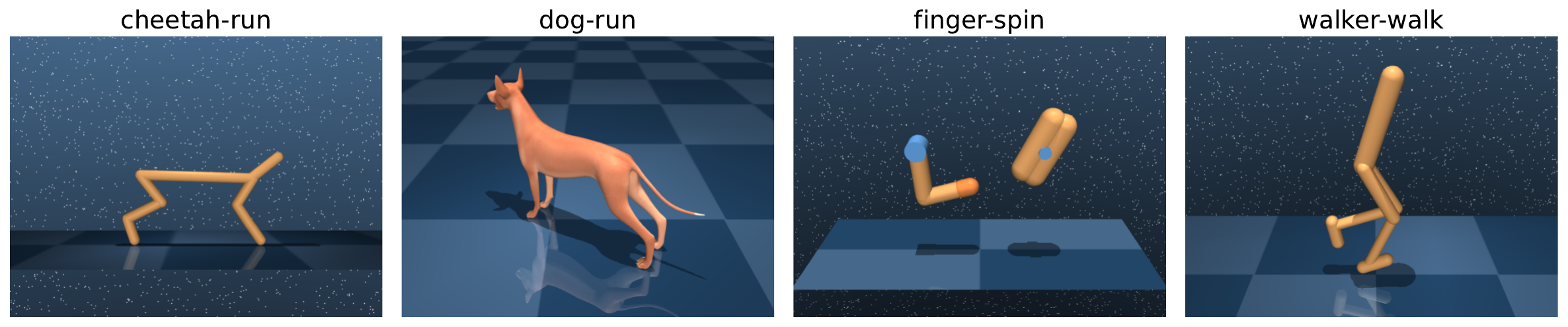}
    \caption{Visual illustrations of four representative tasks from the DMControl Benchmark.}
    \label{fig:dmc_example}
\end{figure}

In this section, we evaluate the effectiveness of spectral representation-based RL algorithms using the widely recognized DeepMind Control (DMControl) Suite \citep{tassa2018deepmind}. Our experiments are designed to answer two primary questions: 1) whether learning spectral representations improves policy optimization, and 2) how different formulations and learning methods for these representations compare against each other and state-of-the-art baselines. To fulfill this, we conduct our evaluation on over 20 distinct tasks from the suite, encompassing both low-dimensional proprioceptive states and high-dimensional visual inputs. Rendered illustrations of these tasks are shown in Figure \ref{fig:dmc_example}. 

\begin{figure}[htbp]
    \centering
    \includegraphics[width=0.95\linewidth]{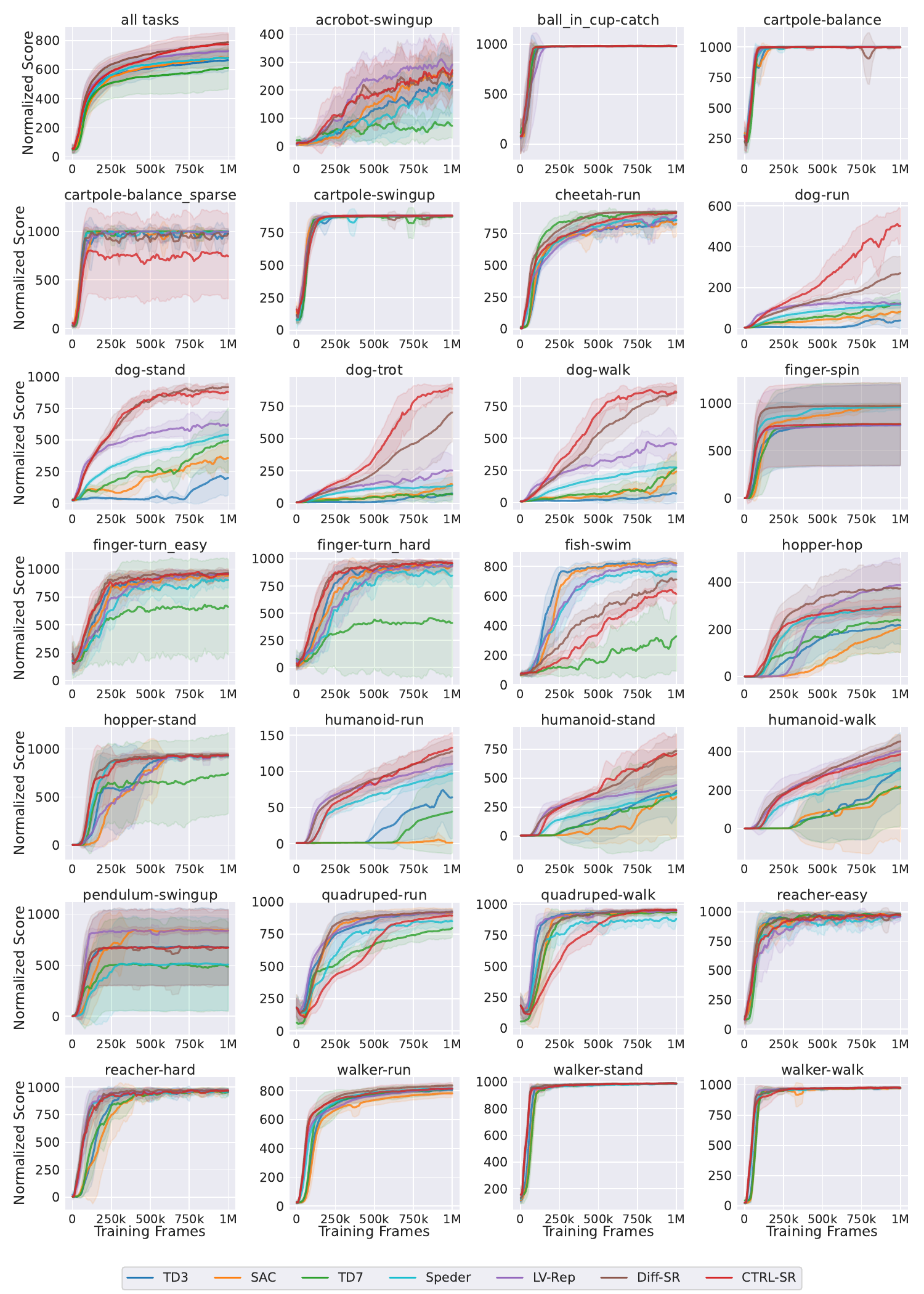}
    \caption{Curves of average episode return on DMControl Suite with proprioceptive inputs, with the first plot depicting the averaged performance across all 27 tasks. All curves are smoothed with a window of 5 for clearer presentation. }
    \label{fig:proprioceptive}
\end{figure}

We implement and compare the following representative instantiations of the spectral representation framework:
\begin{enumerate}
    \item \speder \citep{ren2023spectral}, which employs a linear formulation and learns representations $\phib_\linear$ via spectral contrastive learning \eqref{eq:spectral_contrastive_loss}; 
    \item \lvrep \citep{ren2023latent}, which trains a latent variable spectral representation $\phib_\lv$ through variational learning \eqref{eq:bvae_loss};
    \item \diffsr \citep{shribak2024diffusion}, which utilizes an energy-based formulation $\phib_\ebm$ optimized with score matching \eqref{eq:conditional_score_matching};
    \item \ctrlsr, which improves \ctrl \citep{zhang2022making} by using ranking-based perturbed NCE \eqref{eq:rp_nce} as its representation learning objectives. Note that it also employs the energy-based representations $\phib_\ebm$ instead of the $\phib_\linear$ in \ctrl. 
\end{enumerate}
To ensure a fair and controlled comparison, each method is built upon a fine-tuned implementation of the popular \texttt{TD3} algorithm. We hold all hyperparameters constant across the tasks for each method. This setup isolates the impact of representation learning from that coming from hyperparameter tuning efforts. A detailed description of the hyperparameters and experimental setup is provided in the subsequent sections as well as in Appendix \ref{appsec:implementaion}.

\subsection{Results with Proprioceptive Observations}\label{sec:exp_proprioceptive}

Proprioceptive observations refer to the sensory information that an agent receives about its own internal state, such as the body's position, orientation, and velocity. In DMControl, proprioceptive observations offer a comparatively compact description of the system state, with a dimensionality that scales according to the robot's morphological complexity. For example, the simple \texttt{pendulum-*} task uses a 3-dimensional observation space, comprising the Cartesian coordinates of the pendulum tip and its angular velocity. In contrast, \texttt{dog-*} tasks use a 223-dimensional observation space to encode the complete state of its torso and joints. Collectively, these tasks form a comprehensive evaluation with a wide spectrum of dynamics complexities, allowing for a thorough evaluation of an algorithm's ability to scale to high-dimensional control problems. 

For our baselines, we select two model-free algorithms, \texttt{TD3} \citep{fujimoto2018addressing} and \texttt{SAC} \citep{haarnoja2018soft}. We also include \texttt{TD7} \citep{fujimoto2023sale}, a recent \texttt{TD3} variant that incorporates representation learning, to provide a strong comparison. A consistent training and evaluation protocol is applied to all algorithms and tasks. Each agent interacts with the environment for 1 million frames, with each action being repeated for two consecutive frames (\ie, an action repeat of two frames). The agent's policy is updated every two frames, which amounts to 500,000 gradient steps in total. Performance is evaluated every 10,000 frames by averaging the returns over 10 episodes. To ensure statistical reliability, we report the mean and standard deviation of these evaluation scores across 5 independent runs, each initialized with a different random seed.

Figure \ref{fig:proprioceptive} displays the curves of returns with respect to the environment timesteps, with the first plot providing an aggregated performance summary across all 27 tasks. Although purely model-free algorithms like \texttt{TD3} and \texttt{SAC} can readily solve most tasks, they struggle to scale to environments with high-dimensional tasks. Notably, on complex tasks such as \texttt{dog-trot} and \texttt{humanoid-stand}, they fail to achieve meaningful performances. This outcome highlights a key limitation of model-free RL: its difficulty in leveraging the structural information embedded within the data and in developing temporal abstractions when relying solely on reward signals.

Representation-based methods, with the exception of \texttt{TD7}, consistently outperform their model-free counterpart, \texttt{TD3}. These improvements are particularly pronounced on the more complex $\texttt{dog-*}$ tasks and $\texttt{humanoid-*}$ tasks, where algorithms using spectral representations achieve significant gains. Among the spectral representation algorithms, a clear performance hierarchy emerged. \speder generally performed the poorest, followed by \lvrep, with the two energy-based methods, \diffsr and \ctrlsr, achieving the best results. This is likely due to the inherent limitation of linear spectral representations, as finite-dimensional representations may cause information loss when representing both the dynamics and the value functions. For the latent variable version, although this method is more flexible by using infinite-dimensional representations, the Monte-Carlo approximation used to represent the $Q$-value functions can, however, introduce high variance into training. Finally, the energy-based approach is the most flexible one, since 1) the representations are implicitly infinite-dimensional due to the Gaussian kernel transformation, and 2) it introduces additional learning parameters that are trained by the critic loss, allowing it to offer the best performance among all the methods evaluated.

\begin{figure}[h]
    \centering
    \includegraphics[width=0.9\linewidth]{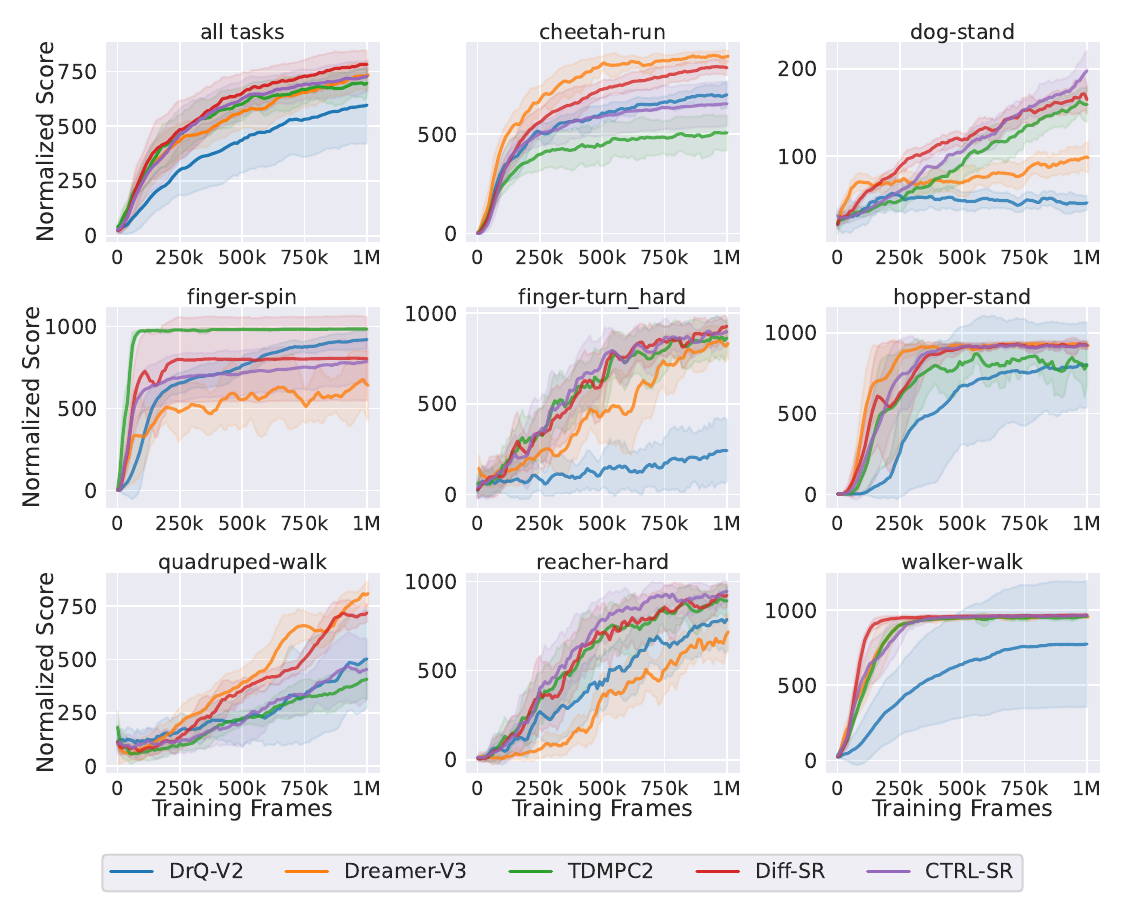}
    \caption{Curves of average episode return on DMControl Suite with visual inputs, with the first plot depicting the averaged performance across all 8 tasks. All curves are smoothed with a window of 5 for clearer presentation.}
    \label{fig:visual}
\end{figure}

\subsection{Results with Visual Observations}
To further demonstrate the benefits of spectral representations, we evaluate our methods on a subset of 8 tasks from the DMControl Suite with visual observations. At each timestep, the agent receives a rendered $84\times 84$ third-person view of its current state instead of receiving the ground-truth proprioceptive state. To compensate for the partial observability inherent in single frames, we follow the standard practice of stacking 3 consecutive frames to form the agent's observation. The combination of high-dimensional visual data and complex dynamics provides a rigorous benchmark for assessing the capabilities of each algorithm. Besides, the action repeat is kept as 2 frames for all algorithms, and the evaluation protocol is the same as that in \secref{sec:exp_proprioceptive}.  

For the baseline algorithms, we include: 1) \texttt{DrQ-V2} \citep{yarats2021mastering}, a competitive model-free visual RL algorithm which employs data augmentation to enhance the visual encoders; 2) \texttt{TDMPC2} \citep{hansen2023td}, a model-based algorithm that performs model predictive control during sampling; and 3) \texttt{Dreamer-V3} \citep{hafner2023mastering}, the state-of-the-art model-based algorithm that employs an RSSM structured model for agent training. 

We implemented the best-performing methods, namely \diffsr and \ctrlsr, for this setting. To handle the visual input, we adopt a convolusional visual encoder to extract a lower-dimensional latent vector from the stacked frames, and the representation networks treat the extracted latent vectors as states. For \ctrlsr, the visual encoder is jointly optimized along with the representation networks, while for \diffsr, we found that joint optimization may lead to representation collapse, a similar phenomenon that is also observed by other representation learning methods \citep{grill2020bootstrap,bardes2021vicreg}. Thus, we decided to train the visual encoder using a variance-regularized reconstruction objective. More details about the implementation can be found in Appendix \ref{appsec:implementaion}.

\begin{wrapfigure}{r}{0.45\textwidth}
  \begin{center}
    \includegraphics[width=1.0\linewidth]{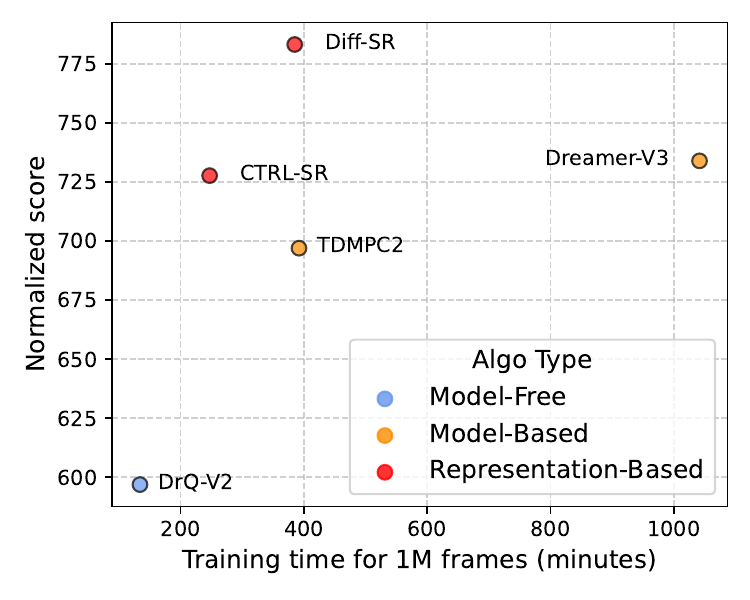}
  \end{center}
  \caption{Runtime comparison on tasks with visual observations.}
    \label{fig:runtime}
\end{wrapfigure}
The learning curves are presented in Figure \ref{fig:visual}, with the first plot presenting an aggregated result across 8 tasks. As a baseline, \texttt{DrQ-V2}, which relies solely on visual data augmentation to regularize the visual encoder, consistently underperforms compared to dynamics-informed methods. In contrast, spectral representation-based methods, despite also being model-free, achieve performance comparable to that of leading model-based algorithms. Notably, they even outperform these approaches on challenging tasks such as \texttt{dog-stand}. Furthermore, Figure \ref{fig:runtime} compares the training time required by each algorithm. Despite utilizing larger network architectures, both \diffsr and \ctrlsr require significantly less training time than model-based competitors, since the representation-based method avoids the costly model-based planning procedure. Collectively, these observations validate that spectral representations provide a powerful and efficient foundation for policy optimization.

\subsection{Ablation Study and Analysis}

\begin{figure}[htbp]
    \centering
    \includegraphics[width=1.0\linewidth]{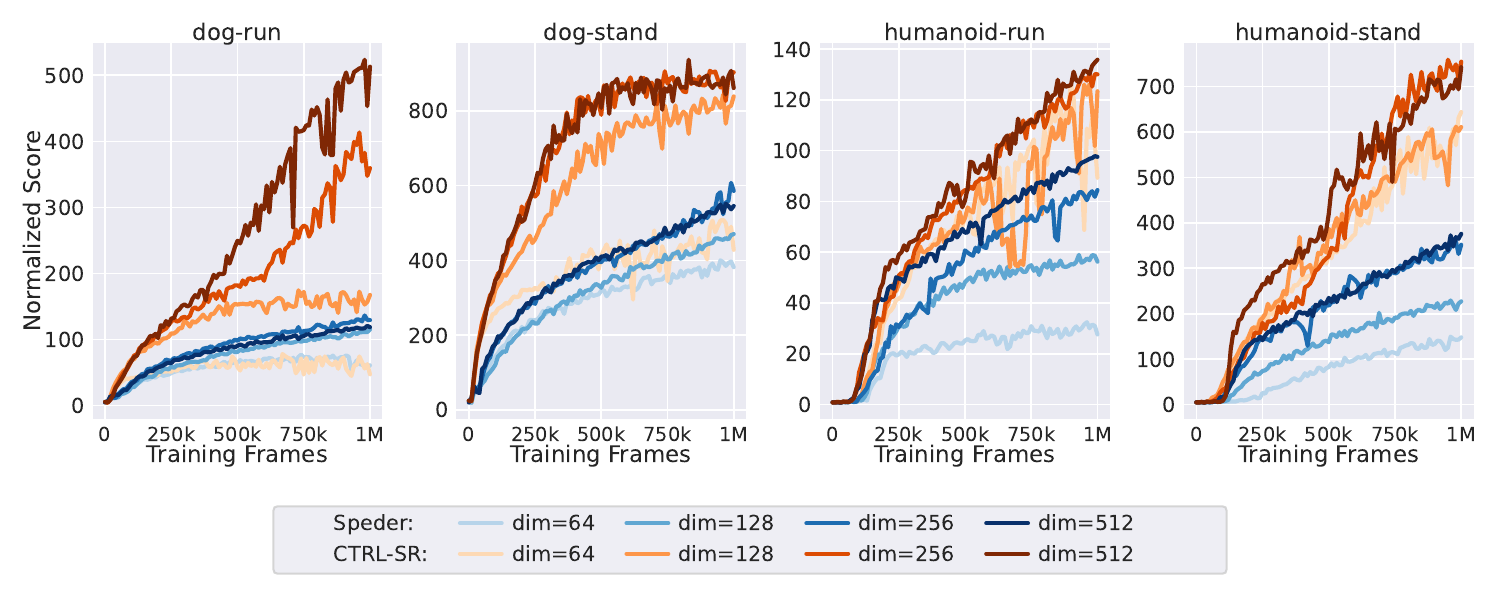}
    \caption{Performance of \speder and \ctrlsr with different representation dimensions. }
    \label{fig:ablation_fdim}
\end{figure}

\paragraph{Representation Dimension.} Despite sharing the linear structure, different formulations in \secref{sec:spectral_representations} demonstrate different scaling behaviors. As shown in figure \ref{fig:ablation_fdim}, we compare \speder and \ctrlsr as representatives for the linear and energy-based formulation, respectively. While both methods improve with a higher representation dimension $d$, the energy-based formulation achieves better final performance under the same capacity and continues to scale up as the dimension increases. We hypothesize this is because energy-based models implicitly learn an infinite-dimensional feature map. While this map is ultimately truncated to a finite dimension using Random Fourier Features, the fact that these frequencies are selected by the critic function better facilitates the downstream policy optimization.

\begin{figure}[htbp]
    \begin{subfigure}[b]{1.0\textwidth}
        \centering
        \includegraphics[width=1.0\textwidth]{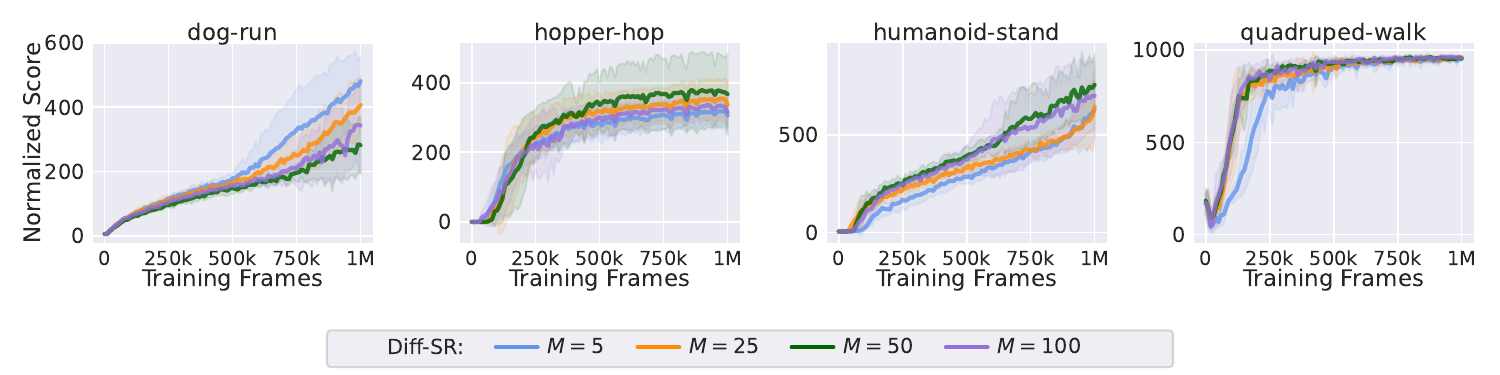}
    \end{subfigure}
    \begin{subfigure}[b]{1.0\textwidth}
        \centering
        \includegraphics[width=1.0\textwidth]{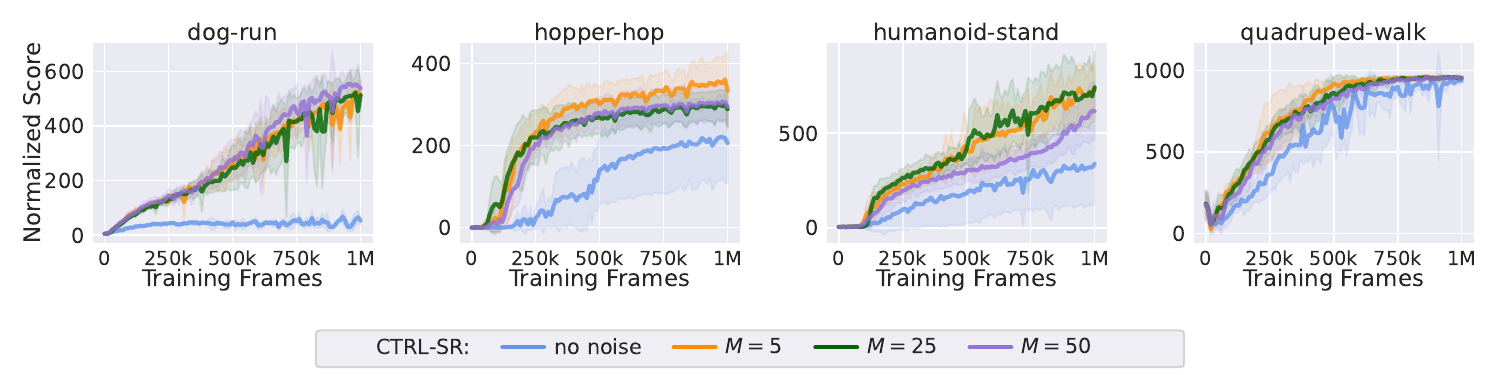}
    \end{subfigure}
    \caption{Performance of \diffsr and  \ctrlsr with different number of noises. }\label{fig:ablation_noise}
\end{figure}

\paragraph{Noise Perturbation. }Both \diffsr and \ctrlsr utilize noise perturbation to improve representation learning, albeit for different reasons. \diffsr trains its network to denoise state transitions, whereas \ctrlsr uses noise to prevent representation collapse due to an overly simple contrastive task. In Figure \ref{fig:ablation_noise}, we investigate how performance changes with varying noise levels. We observe that, generally, there is no significant performance variation across different noise magnitudes. However, for \ctrlsr, performance degrades drastically in the complete absence of noise. This result showcases the efficacy of the noise perturbation technique, particularly for stabilizing the contrastive learning process.

\paragraph{Coupled Training with Critic Objectives.} In our proposed algorithm \ref{alg}, the representation network is trained exclusively using representation learning objectives (line 6). Our theory suggests that representations derived from the transition operator decomposition are sufficient for expressing the Q-value function, thereby eliminating the need to use critic objectives to further tune the network. However, since we are working with finite-dimensional approximations of the spectral representations, it is likely that useful dimensions for representing $Q$-value functions are truncated. Therefore, we also evaluate a variant that couples the critic and representation learning objectives. This approach can be viewed as using the representation objective as a regularization term, while the critic objective is used to train both the representation network and the $Q$-value function. As observed in Figure \ref{fig:ablation_backq}, the variants that combine both objectives demonstrate uniform improvement across all evaluated tasks.

\begin{figure}[t]
    \centering
    \begin{subfigure}[b]{1.0\textwidth}
        \centering
        \includegraphics[width=1.0\textwidth]{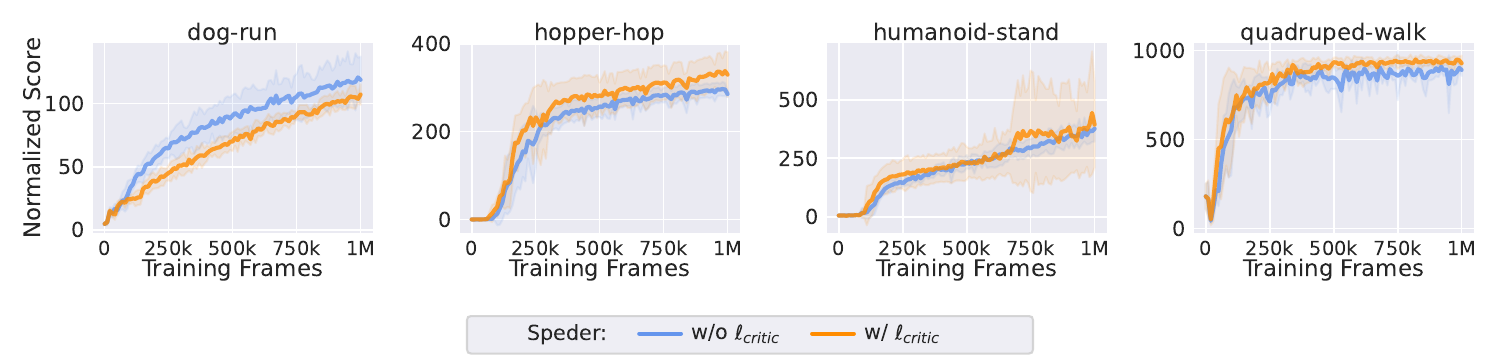}
    \end{subfigure}
    \begin{subfigure}[b]{1.0\textwidth}
        \centering
        \includegraphics[width=1.0\textwidth]{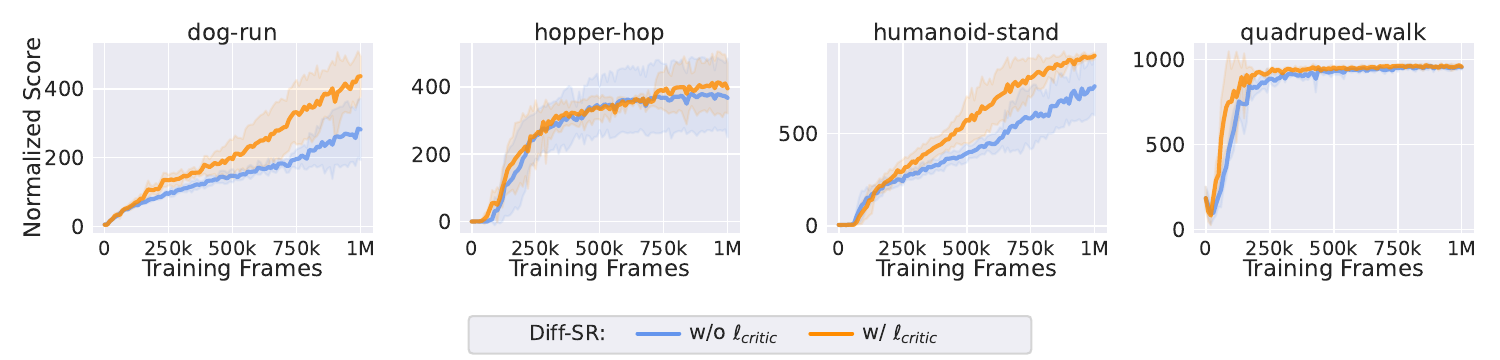}
    \end{subfigure}
    \begin{subfigure}[b]{1.0\textwidth}
        \centering
        \includegraphics[width=1.0\textwidth]{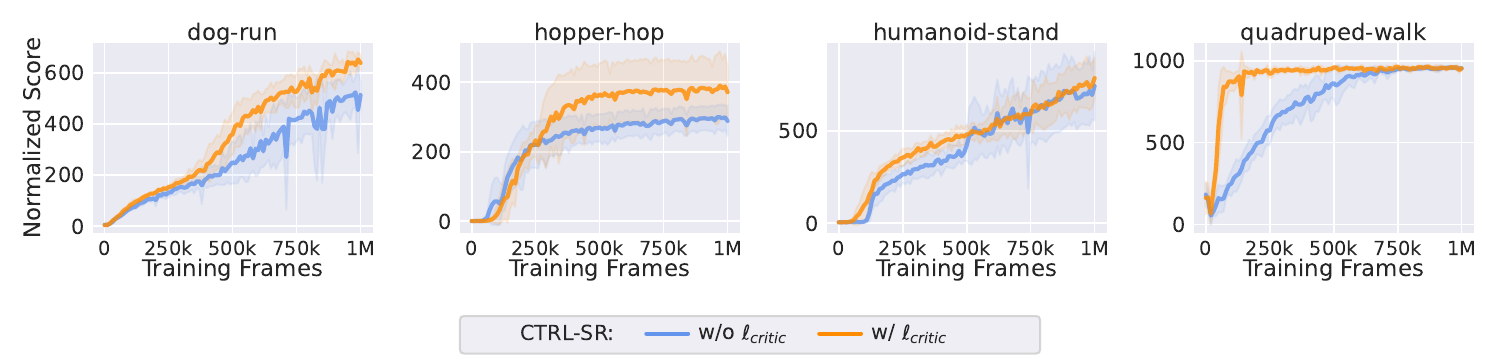}
    \end{subfigure}
    \caption{Performance of \speder, \diffsr, and  \ctrlsr with or without using $\ell_{\text{critic}}$ for representation learning. }\label{fig:ablation_backq}
\end{figure}
\vspace{-2mm}

\section{Related Work}
\subsection{Reinforcement Learning with Representation Learning}

Given the advancements in self-supervised learning (SSL) and the inherent challenges of reinforcement learning (RL) in high-dimensional spaces, there is growing interest in leveraging representation learning to improve RL efficiency. A straightforward application of this approach has been in visual RL, where agents perceive the environment through camera images that often contain redundant information. Many established visual representation techniques, including data augmentation \citep{yarats2021image,yarats2021mastering}, observation reconstruction \citep{yarats2021improving}, and contrastive representation learning \citep{laskin2020curl}, can be readily integrated to derive compact features for the learning. 

In parallel to visual representation learning, a long-standing topic in RL is learning \emph{temporal} representations that encode system behavior and facilitate long-term planning. Among them, successor features \citep{dayan1993improving, kulkarni2016deep,barreto2017successor} and forward-backward representations \citep{touati2021learning} learn to summarize the long-term consequences of actions and can potentially adapt to different reward functions by applying the learned representations. Methods based on bisimulation metrics \citep{ferns2004metrics, gelada2019deepmdp,castro2020scalable,liao2023policy} quantify the behavioral similarity between states and learn representations that group equivalent states together. Inspired by model-based approaches, a large body of work \citep{munk2016learning,lee2020stochastic,guo2020bootstrap,schwarzer2020data,mcinroe2021learning,yu2022mask,tang2023understanding,kim2025self,fujimoto2023sale,fujimoto2025towards,assran2025v} learns representations by predicting future system states, either directly from observations or in a latent space. The spectral representation framework \citep{ren2022free,ren2023latent,zhang2022making,ren2023spectral,shribak2024diffusion} addressed in this paper is closely related to these methods, especially \citet{fujimoto2025towards}, which learn a linear, self-predictive latent representation of states and actions and build the Q-value function upon them. To see this connection, the second term of the spectral contrastive loss in \eqref{eq:spectral_contrastive_loss} is equivalent to predicting the representation of the next state under deterministic MDPs and certain normalization conditions on the representations. The variational learning objective in \eqref{eq:bvae_loss} can also be cast as a non-linear extension of the latent representation prediction problem. Finally, the contrastive learning objective used by \ctrlsr has also been employed to extract temporal abstractions of the system by distinguishing ground-truth state transitions from faked ones \citep{oord2018representation,anand2019unsupervised,stooke2021decoupling,zheng2023texttt}. 

While the learning objectives share many similarities, spectral representations go beyond merely capturing the predictive structure. By exploiting the spectrum structure of the dynamics, spectral representations are distinguished from these alternatives by their unique connection to the $Q$-value functions, which we have demonstrated in \thmref{thm:main} is crucial for exploration design and efficient learning.


\subsection{Model-Based Reinforcement Learning}

Model-based reinforcement learning (RL) learns a model of the environment dynamics from observed state transitions and rewards, and then leverages it to facilitate policy optimization \citep{luo2024survey}. One predominant approach treats the approximated model as a surrogate of the true environment, allowing the agent to generate synthetic trajectories through fictitious interaction \citep{chua2018deep,luo2018algorithmic,janner2019trust,yu2020mopo}. In essence, these synthetic trajectories augment the real-world experience by illustrating counterfactuals, depicting what might have happened if different actions were taken, thereby reducing the demand for trial-and-error in the real world. In terms of modeling, various generative models or architectures have been explored, spanning from Gaussian processes \citep{kamthe2018data}, neural network parameterized deterministic or stochastic models \citep{chua2018deep,luo2018algorithmic,hansen2022temporal,hansen2023td,zhou2411dino}, recurrent state space models (RSSMs) \citep{hafner2019learning,hafner2019dream,hafner2020mastering,hafner2023mastering}, transformers \citep{chen2022transdreamer,micheli2022transformers}, energy-based models \citep{chen2024offline}, and diffusion models \citep{ding2024diffusion,alonso2024diffusion}. Beyond the intricate design of the model architecture, the planning process itself also warrants careful design to mitigate the risk of generating unrealistic trajectories. This is because errors in state prediction can accumulate with long planning horizons, a phenomenon known as the \emph{compounding error}. To manage this error, common strategies include branch rollout \citep{janner2019trust}, early truncation \citep{kidambi2020morel}, or reward penalty \citep{yu2020mopo,sun2023model}. In contrast, our proposed methods adopt a similar way to extract dynamics-informed representations from experiences but completely obviate the need for planning, thereby achieving greater computational efficiency.

Apart from learning the dynamics model with task-specific interaction data, there has been a notable trend in adapting foundation models for dynamics prediction. This approach, first proposed as \textit{world models} \citep{ha2018recurrent}, has been revitalized due to the advances in generative models. V-JEPA 2-AC \citep{assran2025v} and DINO-world \citep{baldassarre2025back} adopt similar approaches, by first pre-training visual encoders on large-scale video or image datasets, and subsequently post-training a predictive model with action-annotated datasets for control and planning. To even eliminate the need for such action-annotated datasets, Genie \citep{bruce2024genie} infers latent actions from unlabelled videos and learns a world model modulated by such latent actions jointly. While substantial progress has been made on the visual pre-training, the dynamics modeling component remains comparatively underexplored. Spectral representations help fill this gap by offering a principled and diverse control-oriented perspective that can be paired with existing high-quality visual representations.

\section{Closing Remarks}

In this survey, we systematically review the advances in spectral representations for efficient reinforcement learning. We elucidate the core design choices in this framework, namely the specific formulation of the spectral representations and the choice of learning algorithm. Focusing on the online reinforcement learning setting, we implement representative methods and compare existing methods under a fair protocol. This provides a quantitative comparison of these algorithms both against one another and against standard baselines.

We believe our review and experiments reveal only a fraction of the potential held by spectral representations. Beyond online reinforcement learning, spectral representations provide a natural way to leverage offline \citep{levine2020offline} or passive \citep{ghosh2023reinforcement} interaction data, facilitating sample-efficient online adaptation. Unlike skill-based pre-training that often requires intricate algorithm design and is heavily dependent on the behavior policy, pre-training spectral representations only requires access to transition data, thus possessing superior transferability and generalization~\citep{ren2023spectral,wang2025learning}. Furthermore, spectral representations can be readily extended to various RL scenarios, such as off-policy evaluation \citep{hu2024primal}, multi-agent RL \citep{ren2024scalable}, and goal-conditional RL \citep{eysenbach2022contrastive}, by identifying the spectral structure within the system dynamics and connecting them to the components in RL pipeline. 

Outside the scope of reinforcement learning, spectral representation also casts a novel perspective on modular designs of modern generative models like large language models (LLMs) and diffusion models. As these models are also pre-trained on large-scale, unlabeled datasets, they also yield instrumental representations for downstream tasks like fine-tuning or post-training. We hope this survey can inspire further exploration of this perspective and foster the design of more efficient and robust algorithms.

\section*{Acknowledgements}
This paper is supported by NSF AI institute: 2112085,
NSF ECCS-2401390, NSF ECCS-2401391, ONR
N000142512173, and NSF IIS2403240.

\bibliographystyle{plainnat}
\bibliography{main}

\newpage
\appendix
\section{Online Exploration with Spectral Representations}
Following \citet{guo2023provably}, with the estimation of $\hat{\phib}_n$ at the $n$-th episode, we introduce the bonus term $\hat{b}_n(s, a)$ to implement the principle of optimism in the face of uncertainty. However, since our representation lies in the kernel space $\Hcal_k$, we follow \citet{yang2020function} and define the bonus function in terms of the kernel:
\begin{equation}\label{eq:bonus}
    \hat{b}_n(s, a)=\alpha_n\lambda^{-1/2}\cdot\left(\hat{K}^{(n)}_{(s, a), (s, a)}-\hat{K}^{(n)\top}_{(s,a)}\left(\hat{K}^{(n)}+\lambda I\right)\hat{K}^{(n)}_{(s, a)}\right)^{1/2},
\end{equation}
where
\begin{equation}
    \begin{aligned}
        \hat{K}^{(n)}_{(s, a),(s, a)}&=\langle\hat{\phib}_n(s, a),\hat{\phib}_n(s_i, a_i)\rangle_{\Hcal_k}\in\RR,\\
        \hat{K}^{(n)}_{(s, a)}&=[\langle\hat{\phib}_n(s, a),\hat{\phib}_n(s_i, a_i)\rangle_{\Hcal_k}]_{i\in [|\Dcal_n|]}\in \RR^{|\Dcal_n|},\\
        \hat{K}^{(n)}&=[\langle\hat{\phib}_n(s_i, a_i),\hat{\phib}_n(s_j, a_j)\rangle_{\Hcal_k}]_{i,j\in [|\Dcal_n|]}\in \RR^{|\Dcal_n|\times|\Dcal_n|}.\\
    \end{aligned}
\end{equation}
The bonus function is added to the extrinsic reward and used for planning in the approximated dynamics to learn optimistic $Q$-values that upper-bound the true $Q$-values. The optimistic version of our algorithm, which incorporate such bonus functions, is defined in Algorithm \ref{alg_optimistic}. Note that this algorithm is only leveraged for theoretical analysis, but not for actual implementations.

\begin{algorithm}[h]
\caption{Optimistic Online Reinforcement Learning with Spectral Representations}
\label{alg_optimistic}
\textbf{Initialize}: replay buffer $\Dcal_0=\emptyset$, $\pi_0=\Ucal(\Acal)$

\begin{algorithmic}[1]
\FOR{episode $n=1, 2, \cdots, N$}
    \STATE Collect the transitions $(s_h, a_h, s_{h+1}, r_h)$ following $a_h\in\Ucal(\Acal)$, $s_{h+1}\sim\PP(\cdot|s_h, a_h)$, $r_h=r(s_h, a_h)$
    \STATE $\Dcal_n\leftarrow \Dcal_{n-1}\cup\{(s_h, a_h, s_{h+1}, r_h)\}_{h=1}^H$
    \STATE Learn the spectral representations $\hat{\phib}_n$ and $\hat{\mub}_n$ using $\Dcal_n$
    \STATE Specify the exploration bonus function $\hat{b}(s, a)$ as \eqref{eq:bonus}
    \STATE Update $\pi_n=\argmax_\pi V^\pi_{\hat{\PP},r+\hat{b}_n}$, where $V^\pi_{\hat{\PP}_n,r+\hat{b}_n}$ is the value function of $\pi$ with approximated dynamics $\hat{\PP}_n=\inner{\hat{\phib}_n}{\hat{\mub}_n}_{\Hcal_k}$ and augmented reward function $r+\hat{b}_n$
\ENDFOR

\textbf{Return} $\pi_1, \pi_2, \ldots, \pi_N$
\end{algorithmic}
\end{algorithm}

\section{Implementation Details}\label{appsec:implementaion}
In this section, we introduce the implementation details of each algorithm in our empirical evaluation section. In \tabref{tab:proprio_common_hyperparam}, we list the common hyperparameters that are shared across most of the algorithms in this paper. For details about each specific algorithm, we will detail below. 

\begin{table}[htbp]
\centering
\caption{Hyperparameters shared across different algorithms. }
\label{tab:proprio_common_hyperparam}
\begin{tabular}{c|cc}
\toprule
\textbf{Parameter} & \textbf{Proprioceptive} & \textbf{Visual}\\
\midrule
frame skip & 2 & 2\\
frame stack & 1 & 3\\
buffer size & 1,000,000 & 1,000,000\\
batch size & 512 & 256\\
train frames & 1,000,000 & 1,000,000\\
warmup frames & 10,000 & 4000\\
discounting factor $\gamma$ & 0.99 & 0.99\\
$n$-step TD & 1 & 3\\

\bottomrule
\end{tabular}
\end{table}

\subsection{Implementation Details of Baseline Algorithms}
\paragraph{Configurations of \texttt{TD3}} \citep{fujimoto2018addressing}. Due to its simplicity and strong empirical performance, we re-implemented \texttt{TD3} in our code repository and built all spectral representation-based algorithms on \texttt{TD3}. The hyperparameters that we used to obtain the results are listed in \tabref{tab:td3_hyperparam}. Note that we tuned certain design choices (such as target network update strategy) and hyperparameters based on our observation on the DMControl suite, and therefore, the final choice may slightly differ from the original implementation. 
\begin{table}[htbp]
\centering
\caption{Hyperparameters specific for \texttt{TD3}.}
\label{tab:td3_hyperparam}
\begin{tabular}{c|c}
\toprule
\textbf{Parameter} & \textbf{Value}\\
\midrule
actor network & \texttt{MLP(dim(S), 512, 512, 512, dim(A))}\\
critic network & \texttt{MLP(dim(S)+dim(A), 512, 512, 512, 1)}\\
learning rate & 0.0003\\
soft update & True\\
soft update rate $\tau$ & 0.005\\
target update interval & 1\\
actor update interval & 1\\
policy noise & 0.2\\
noise clip & 0.3\\
exploration noise & 0.2\\
\bottomrule
\end{tabular}
\end{table}

\paragraph{Configurations of \texttt{SAC}} \citep{haarnoja2018soft}. Compared to \texttt{TD3}, \texttt{SAC} introduces the principle of maximum entropy into reinforcement learning and regularizes the traditional RL objective with entropy. Empirically, entropy regularization encourages the policy to explore the environment, thus potentially leading to improved performance. The hyperparameters of \texttt{SAC} are listed in \tabref{tab:sac_hyperparam}. We enable the automatic tuning of the coefficient of the entropy term and set the target entropy to $-\text{dim}(\Acal)$ by default. 
\begin{table}[htbp]
\centering
\caption{Hyperparameters specific for \texttt{SAC}.}
\label{tab:sac_hyperparam}
\begin{tabular}{c|c}
\toprule
\textbf{Parameter} & \textbf{Value}\\
\midrule
actor network & \texttt{MLP(dim(S), 512, 512, 512, 2dim(A))}\\
critic network & \texttt{MLP(dim(S)+dim(A), 512, 512, 512, 1)}\\
learning rate & 0.0003\\
soft update & True\\
soft update rate $\tau$ & 0.005\\
target update interval & 1\\
auto entropy & True, with target as \texttt{-dim(A)}\\
\bottomrule
\end{tabular}
\end{table}

\paragraph{Configurations of \texttt{TD7}} \citep{fujimoto2023sale}. \texttt{TD7} incorporates several improvements over \texttt{TD3}, including representation learning by embedding prediction, 2) policy checkpointing, behavior cloning regularization for offline scenarios, and LAP replay buffer \citep{fujimoto2020equivalence}. Different from how we leverage spectral representations, TD7 feeds the representations alongside with the raw observations and actions into the value networks. We re-implemented the TD7 algorithm for DMControl tasks and used the recommended hyperparameters in the author-provided implementation, which are listed in \tabref{tab:td7_hyperparam}. 
\begin{table}[htbp]
\centering
\caption{Hyperparameters specific for \texttt{TD7}.}
\label{tab:td7_hyperparam}
\begin{tabular}{c|c}
\toprule
\textbf{Parameter} & \textbf{Value}\\
\midrule
embedding dimension & 256\\
hidden dimension & 256\\
learning rate & 0.0003\\
target update interval & 250\\
actor update interval & 1\\
policy noise & 0.2\\
noise clip & 0.5\\
exploration noise & 0.1\\
LAP $\alpha$ & 0.4\\
LAP max priority reset interval & 250\\
\bottomrule
\end{tabular}
\end{table}

\paragraph{Configurations of \texttt{DrQ-V2}} \citep{yarats2021mastering}. \texttt{DrQ-V2} is a model-free reinforcement learning algorithm designed to efficiently solve complex control tasks by directly learning from raw pixel data. The core innovation of \texttt{DrQ-V2} is to apply image augmentation to the raw observations of the agents to regularize the visual encoders, therefore preventing divergent optimization due the challenge of learning from pixels. Although the official implementation of \texttt{DrQ-V2} uses different sets of hyperparameters for varying task difficulties, we used identical hyperparameters for all tasks in our experiments. The hyperparameters are listed in \tabref{tab:drqv2_hyperparam}. 
\begin{table}[htbp]
\centering
\caption{Hyperparameters specific for \texttt{DrQ-V2}.}
\label{tab:drqv2_hyperparam}
\begin{tabular}{c|c}
\toprule
\textbf{Parameter} & \textbf{Value}\\
\midrule
actor hidden dimension & 1024\\
critic hidden dimension & 1024\\
embedding dimension & 50\\
actor std schedule & \texttt{linear(1.0, 0.1, 500000)}\\
actor std clip & 0.3\\
critic loss & Mean Squared Error (MSE) \\ 
learning rate & 0.0001\\
network update interval & 2 \\
network soft update rate & 0.01\\
\bottomrule
\end{tabular}
\end{table}

\paragraph{Configurations of \texttt{Dreamer-V3}} \citep{hafner2023mastering}. \texttt{Dreamer-V3} is a general and highly performant model-based reinforcement learning algorithm. It core concept is to first learn an RSSM that predicts the future latent states based on past observations and actions, thus enabling efficient simulation within the latent space. The agent then learns optimal behaviors through imaginary interaction with this learned simulation, rather than through slow and costly real-world interaction. Such a model-based approach makes it incredibly data-efficient compared to model-free methods. For our experiments, we used existing PyTorch implementation of \texttt{Dreamer-V3} provided in \url{https://github.com/NM512/dreamerv3-torch} and adapted it for our evaluation protocol. 

\paragraph{Configurations of \texttt{TDMPC2}} \citep{hansen2023td}. \texttt{TDMPC2} is another representative model-based reinforcement learning algorithm. Compared to \texttt{Dreamer-V3}, \texttt{TDMPC2} learns a light-weight latent space dynamics model by predicting the representations of the next observations using the representation of current observations and actions. Another innovation of \texttt{TDMPC2} is that, it incorporates model predictive control during sampling to actively explore the environments based on the current estimation of the model. This design significantly improves the undirectional exploration strategy used in common RL algorithms, and boosts sample efficiency. Similarly, we re-used the author provided PyTorch implementation and configurations in \url{https://github.com/nicklashansen/tdmpc2}, and only made modifications to adapt it for our evaluation protocol. 

\subsection{Implementation Details of Spectral Representation-Based Algorithms}
In this section we disclose the details about the spectral representation-based algorithms. We build all these algorithms upon \texttt{TD3}, with the only difference being that the critic networks take the spectral representations as input. Therefore, we will primarily focus on the representation learning part. Besides, since many of our algorithms involve perturbations, we normalize the observation by its mean and standard deviation to keep it on a similar scale of the Gaussian noise. 

In our implementation, we extensively use \texttt{ResidualMLP}, a variant of \texttt{MLP} layers with residual connections, since we find it consistently demonstrates better plasticity and expressiveness, echoing the recent findings by \citet{nauman2024bigger} and \citet{lee2024simba}. A diagram of this architecture is provided in Figure \ref{fig:residual}.

\begin{figure}
    \centering
    \includegraphics[width=0.8\linewidth]{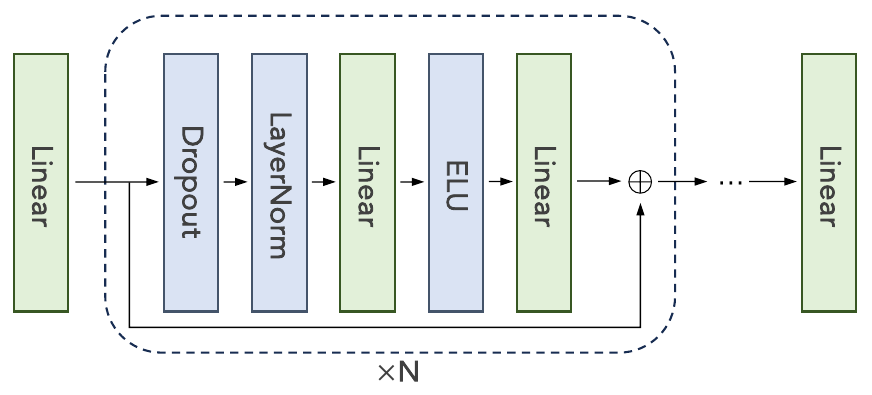}
    \caption{Illustration of the residual MLP network.}
    \label{fig:residual}
\end{figure}

\paragraph{Configurations of \speder}. \texttt{Speder} employs a linear spectral formulation and uses the spectral contrastive learning to optimize the representations. Since in essence it is still performing contrastive learning, we also employ the noise perturbation technique for \texttt{Speder} to regularize the contrastive learning from degeneration. We also use a target network for the representation networks, whose update rate is $0.01$. The configurations of \texttt{Speder} are listed in \tabref{tab:speder_hyperparam}. 
\begin{table}[htbp]
\centering
\caption{Hyperparameters specific for \texttt{Speder}.}
\label{tab:speder_hyperparam}
\begin{tabular}{c|c}
\toprule
\textbf{Parameter} & \textbf{Value}\\
\midrule
representation dimension $d_r$ & 512\\
representation learning rate & 0.000005\\
representation target update rate & 0.01\\
noise embedding dimension $d_n$ & 128\\
architecture of $\varphib$ & \texttt{ResidualMLP(dim(S)+dim(A),512,512,$d_r$)}\\
architecture of $\nub$ & \texttt{ResidualMLP(dim(S)+$d_n$,512,512,$d_r$)}\\
number of noise levels & 25\\
reward prediction loss weight & 0.1\\
\bottomrule
\end{tabular}
\end{table}

\paragraph{Configurations of \lvrep}. \lvrep uses a standard VAE architecture but with a learnable prior distribution. To prevent regularizing the posterior distribution with a bad prior in the initial stage of training, \citet{hafner2020mastering} applies decoupled and asymmetric KL regularization for the posterior and prior. However, in our experiments, we found this provides marginal benefits to the performance, and therefore we a balanced regularization by default. The configurations of \lvrep are listed in \tabref{tab:lvrep_hyperparam}.
\begin{table}[htbp]
\centering
\caption{Hyperparameters specific for \lvrep.}
\label{tab:lvrep_hyperparam}
\begin{tabular}{c|c}
\toprule
\textbf{Parameter} & \textbf{Value}\\
\midrule
representation dimension $d_r$ & 512\\
representation learning rate & 0.0001\\
representation target update rate & 0.005\\
distribution of the latent & diagonal Gaussian\\
architecture of decoder & \texttt{MLP(2*dim(S)+dim(A),512,512,512,$2*d_r$)}\\
architecture of decoder & \texttt{MLP($d_r$,512,512,512,dim(S))}\\
architecture of the prior & \texttt{MLP(dim(S)+dim(A),512,512,$d_r$)}\\
minimum log std & -20\\
maximum log std & 2.0\\
coefficient for KL $\beta$ & 0.1\\
reward prediction loss weight & 0.05\\
\bottomrule
\end{tabular}
\end{table}

\paragraph{Configurations of \diffsr (with proprioceptive observation)}. In \diffsr, we are learning a series of score functions of noise-perturbed distributions and extracting the spectral representations from $\nabla_{\stil'}\log \PP(\stil'|s, a; \beta)=\varphib_\theta(s, a)^\top\nabla_{\stil'}\nub_\theta(\stil'; \beta)$. Instead of parameterizing $\nub_\theta(\tilde{s}'; \beta)$ and taking the second order for optimization, we directly parameterize $\nabla_{\tilde{s}'}\nub_\theta(\tilde{s}'; \beta)$ as a neural network with output dimension $d_r\times\dim(\Scal)$, thus its inner product with $\varphib(s, a)$ gives the score estimation. Finally, the configurations of \diffsr are provided in \tabref{tab:diffsr_hyperparam}. 
\begin{table}[htbp]
\centering
\caption{Hyperparameters specific for \diffsr.}
\label{tab:diffsr_hyperparam}
\begin{tabular}{c|c}
\toprule
\textbf{Parameter} & \textbf{Value}\\
\midrule
representation dimension $d_r$ & 512\\
representation learning rate & 0.0001\\
representation target update rate & 0.005\\
noise embedding dimension $d_n$ & 128\\
random Fourier feature $N$ & 512\\
architecture of $\varphib$ & \texttt{ResidualMLP(dim(S)+dim(A),512,512,512,$d_r$)}\\
architecture of $\nabla\nub$ & \texttt{ResidualMLP(dim(S)+$d_n$,512,512,512,$d_r\times$dim(S))}\\number of noise levels & 50\\
reward prediction loss weight & 0.1\\
\bottomrule
\end{tabular}
\end{table}

\paragraph{Configurations of \ctrlsr (with proprioceptive observation)}. Similar to \diffsr, we conduct noise contrastive learning across different levels of noise perturbations. The configurations of \ctrlsr are listed in \tabref{tab:ctrlsr_hyperparam}.
\begin{table}[htbp]
\centering
\caption{Hyperparameters specific for \ctrlsr.}
\label{tab:ctrlsr_hyperparam}
\begin{tabular}{c|c}
\toprule
\textbf{Parameter} & \textbf{Value}\\
\midrule
representation dimension $d_r$ & 512\\
representation learning rate & 0.0001\\
representation target update rate & 0.005\\
noise embedding dimension $d_n$ & 128\\
random Fourier feature $N$ & 512\\
architecture of $\varphib$ & \texttt{ResidualMLP(dim(S)+dim(A),512,512,$d_r$)}\\
architecture of $\nabla\nub$ & \texttt{ResidualMLP(dim(S)+$d_n$,512,512,$d_r$)}\\number of noise levels & 25\\
reward prediction loss weight & 0.1\\
\bottomrule
\end{tabular}
\end{table}

\paragraph{Adaptation of \ctrlsr and \diffsr for visual observations.} To address the challenge posed by the high dimensionality of visual observations, we employ a hierarchical strategy for both \ctrlsr and \diffsr. Instead of performing contrastive learning or score matching directly in the raw pixel space, we first use a lightweight image encoder $E_\theta$ to project the images into compact latent embeddings, which are then used for representation learning. Due to frame stacking, the final embedding used for representation learning is a concatenation of the latent embeddings of each individual frame. For \ctrlsr, this approach works effectively. However, for \diffsr, it may converge to a trivial constant embedding for all possible $(s, a)$ and $s'$. To prevent such representation collapse, we introduce an additional reconstruction and variance-preserving objective for \diffsr:
\begin{equation}
    \begin{aligned}
        \ell_{\text{recon}}(\theta)&=\mathbb{E}_{o\sim\Dcal}\left[\|D_\theta(E_\theta(o)) - o\|^2\right],\\
        \ell_{\text{vp}}(\theta)&=\mathbb{E}_{\{o_i\}_{i=1}^B\sim\Dcal}\left[\frac 1d\sum_{j=1}^d\max(0, 1-\texttt{std}(\{E_\theta(o_i)_j\}_{i=1}^B))\right],
    \end{aligned}
\end{equation}
where $d_r$ is the reprensetation dimension, $E_\theta$ and $D_\theta$ are the image encoder and decoder, respectively. Furthermore, we incorporate image augmentations following \texttt{DrQ-V2}, which helps the visual encoder develop richer and more robust representations. Specifically, for \ctrlsr, we apply image augmentations to both the observation and next observation during contrastive learning, while for \diffsr, augmentations are applied when training the encoder by feeding augmented images and reconstructing the corresponding clean images. Additional configuration details are provided in \tabref{tab:ctrlsr_obs_hyperparam} and \tabref{tab:diffsr_obs_hyperparam}. 
\begin{table}[htbp]
\centering
\caption{Hyperparameters for \ctrlsr with visual observations.}
\label{tab:ctrlsr_obs_hyperparam}
\begin{tabular}{c|c}
\toprule
\textbf{Parameter} & \textbf{Value}\\
\midrule
actor hidden dimension & 1024\\
critic hidden dimension & 1024\\
actor learning rate & 0.0003\\
critic learning rate & 0.0003\\
number of noise levels & 50\\
representation dimension $d_r$ & 512\\
representation learning rate & 0.0001\\
representation target update rate & 0.0005\\
encoder dimension $d_e$ & 512\\
noise embedding dimension $d_n$ & 128\\
random Fourier feature $N$ & 512\\
architecture of $\varphib$ & \texttt{ResidualMLP($d_e$+dim(A),512,512,$d_r$)}\\
architecture of $\nub$ & \texttt{ResidualMLP($d_e+d_n$,512,512,$d_r$)}\\
reward prediction loss weight & 0.1\\
weight decay & 0.0001\\
\bottomrule
\end{tabular}
\end{table}
\begin{table}[htbp]
\centering
\caption{Hyperparameters for \diffsr with visual observations.}
\label{tab:diffsr_obs_hyperparam}
\begin{tabular}{c|c}
\toprule
\textbf{Parameter} & \textbf{Value}\\
\midrule
actor hidden dimension & 1024\\
critic hidden dimension & 1024\\
actor learning rate & 0.0003\\
critic learning rate & 0.0003\\
number of noise levels & 50\\
representation dimension $d_r$ & 512\\
representation learning rate & 0.0001\\
representation target update rate & 0.01\\
encoder dimension $d_e$ & 256\\
noise embedding dimension $d_n$ & 128\\
random Fourier feature $N$ & 512\\
architecture of $\varphib$ & \texttt{ResidualMLP($d_e$+dim(A),1024x4,$d_r$)}\\
architecture of $\nub$ & \texttt{ResidualMLP($d_e+d_n$,1024x4,$d_r$)}\\
reward prediction loss weight & 10.0\\
weight decay & 0.0001\\
$\ell_{\text{recon}}$ loss weight & 1.0\\
$\ell_{\text{vp}}$ loss weight & 0.1\\
\bottomrule
\end{tabular}
\end{table}

\end{document}